\documentclass{article}

\PassOptionsToPackage{numbers}{natbib}

\usepackage[preprint]{neurips_2022}

\usepackage[utf8]{inputenc} 
\usepackage[T1]{fontenc}    
\usepackage{hyperref}       
\usepackage{url}            
\usepackage{booktabs}       
\usepackage{amsfonts}       
\usepackage{nicefrac}       
\usepackage{microtype}      
\usepackage{xcolor}         
\usepackage{amsmath}
\usepackage{amssymb}
\usepackage{graphicx}
\usepackage{siunitx}
\usepackage{wrapfig}
\usepackage{makecell}
\usepackage{amsthm}

\newtheorem{lemma}{Lemma}
\newtheorem{prop}{Proposition}

\title{Adaptive Gradient Methods at the Edge of Stability}

\author{%
    Jeremy Cohen\thanks{Work done as student researcher.  Correspondence to \{jeremycohen@cmu.edu, gilmer@google.com\}.} \quad  Behrooz Ghorbani \quad Shankar Krishnan \quad Naman Agarwal \And \quad Sourabh Medapati \quad Michal Badura \quad Daniel Suo \quad David Cardoze \And \quad Zachary Nado \quad George E. Dahl \quad Justin Gilmer \\ \\ Google Research, Brain Team
}

\begin{document}

\maketitle

\begin{abstract}
   Very little is known about the training dynamics of adaptive gradient methods like Adam in deep learning.  In this paper, we shed light on the behavior of these algorithms in the full-batch and sufficiently large batch settings.  Specifically, we empirically demonstrate that during full-batch training, the maximum eigenvalue of the \emph{preconditioned} Hessian typically equilibrates at a certain numerical value --- the stability threshold of a gradient descent algorithm.  For Adam with step size $\eta$ and $\beta_1 = 0.9$, this stability threshold is $38/\eta$.  Similar effects occur during minibatch training, especially as the batch size grows.  Yet, even though adaptive methods train at the “Adaptive Edge of Stability” (AEoS), their behavior in this regime differs in a significant way from that of non-adaptive methods at the EoS.  Whereas non-adaptive algorithms at the EoS are blocked from entering high-curvature regions of the loss landscape, adaptive gradient methods at the AEoS keep advancing into high-curvature regions, while adapting the preconditioner to compensate.  Our findings can serve as a foundation for the community’s future understanding of adaptive gradient methods in deep learning.
\end{abstract}

\begin{figure}[h]
    \begin{center}
        \includegraphics[width=14cm]{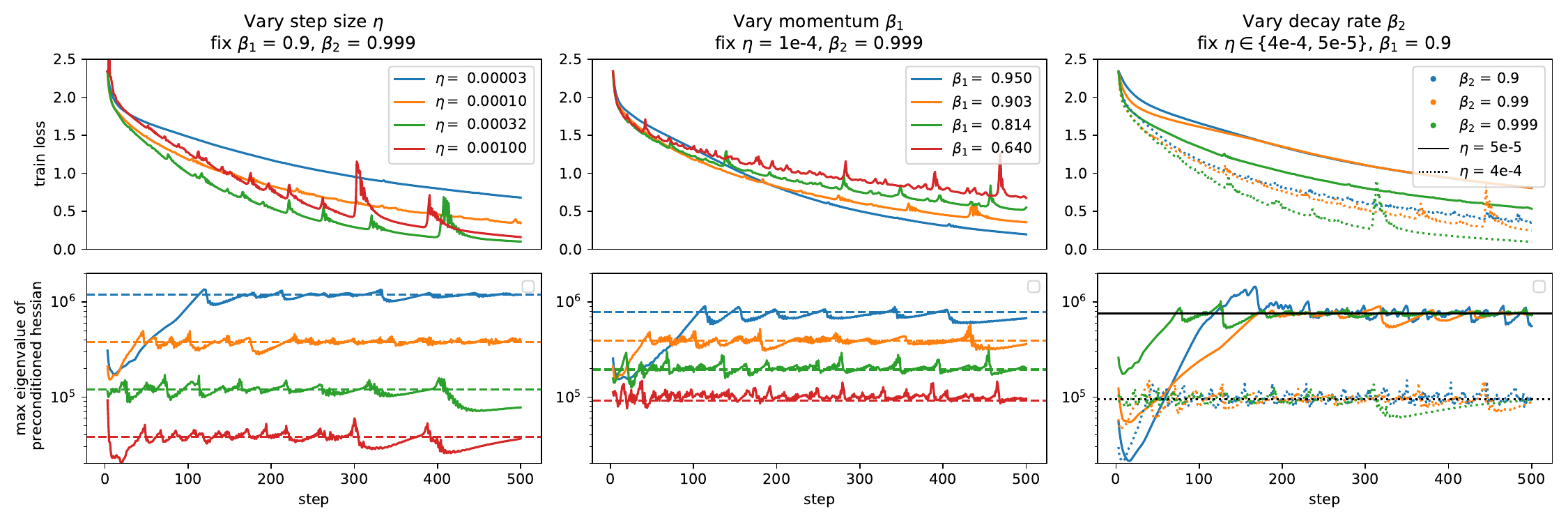}
        \caption{\textbf{Full-batch Adam trains at the Adaptive Edge of Stability (AEoS)}.  We train a fully-connected network on CIFAR-10 using full-batch Adam with various hyperparameters $\eta, \beta_1, \beta_2$.  Observe that the maximum eigenvalue of the preconditioned Hessian equilibrates at the numerical value $\frac{(2 + 2 \beta_1)}{(1 - \beta_1) \eta}$, which is drawn as a dashed horizontal line.  However, in contrast to the non-adaptive EoS, the maximum eigenvalue of the raw  Hessian usually keeps rising at the AEoS (see Figure \ref{fig:hessian-rises}).}
    \label{fig:adam}
    \end{center}
\end{figure}

\newpage 

\section{Introduction}

Neural networks are often trained using adaptive gradient methods such as Adam \citep{kingma2014adam}, rmsprop \citep{tieleman2012divide}, or Adafactor \citep{shazeer2018adafactor}.
These algorithms are variants of preconditioned gradient descent where the preconditioner is constantly \emph{adapting} to gradients from recent steps.
Despite the popularity of these algorithms, their behavior when training neural networks remains poorly understood.

Meanwhile, a line of recent work \citep{nar2018step, wu2018how, jastrzkebski2018relation, jastrzebski2020break, cohen2021gradient, gilmer2021loss} has shed light on the neural network training dynamics of \emph{non-adaptive} gradient descent (both with and without momentum) in, at least, the full-batch and (sufficiently) large-batch regimes.
These works have empirically demonstrated that \emph{dynamical instability} plays a key role in the training process.
When training neural networks, gradient descent is constantly attracted to regions of parameter space with increasingly high curvature \citep{jastrzkebski2018relation, jastrzebski2020break, cohen2021gradient}; yet, up to a quadratic Taylor approximation, gradient descent is an unstable dynamical system in regions where the curvature in any direction exceeds a certain threshold --- the optimizer cannot linger for long in any such region without being expelled \citep{nar2018step, wu2018how, jastrzebski2020break}.
How is this tension resolved?
In the full-batch special case, gradient descent spends the bulk of training in a regime called the Edge of Stability (EoS) \citep{cohen2021gradient} in which the \emph{sharpness} --- the maximum eigenvalue of the training Hessian --- hovers right at, or just above, the stability threshold.
At the EoS, gradient descent would still be moving into regions of higher curvature were it not being constantly repelled from these high-curvature regions by unstable dynamics.
As we confirm below, these findings also generalize to preconditioned gradient descent (with a static preconditioner).

However, it has not been clear whether these findings have relevance for \emph{adaptive} gradient methods.
Because of adaptive preconditioning, adaptive gradient methods do not evolve as linear recurrences on the local quadratic Taylor approximation, and thus it is not clear why their local stability would be well-modeled by an eigenvalue condition.

In this paper, we demonstrate that the EoS phenomenon does, in fact, carry over to the adaptive setting.
Our key empirical finding is that throughout training, the short-term stability behavior of, say, Adam is well-approximated by that of ``frozen Adam'' --- a version of Adam in which the preconditioner is frozen at its current value.
On the local quadratic Taylor approximation, ``frozen Adam'' \emph{does} evolve as a linear recurrence, and is unstable whenever the maximum eigenvalue of the preconditioned Hessian (the \emph{preconditioned sharpness}) exceeds the stability threshold of EMA-style heavy ball momentum, which is $\frac{2 + 2 \beta_1}{\eta (1 - \beta_1)}$.
Indeed, we observe that during full-batch training by real Adam, the preconditioned sharpness equilibrates at this precise numerical value (Figure \ref{fig:adam}).
During minibatch training, similar phenomena occur, paralleling the situation with non-adaptive optimizers.

However, even though adaptive gradient methods train at the ``Adaptive Edge of Stability'' (AEoS), their behavior in this regime differs in a significant way from that of non-adaptive methods in the non-adaptive EoS regime: whereas non-adaptive optimizers in the non-adaptive EoS regime are blocked from accessing high-curvature regions of the loss landscape, we find that adaptive gradient methods at the AEoS can and do enter these high-curvature regions via their ability to adapt the preconditioner.
This is especially liable to occur if the step size is small or the preconditioner decay factor (e.g. Adam's $\beta_2$) is small.
Thus, adaptive gradient methods sometimes lack the implicit inductive bias \citep{neyshabur2014search} that blocks non-adaptive methods from converging to high-curvature solutions.

\section{Notation and Background}
\label{sec:background}

Optimization algorithms aim to minimize an objective function $f: \mathbb{R}^p \to \mathbb{R}$ by producing a sequence of iterates $\{\mathbf{x}_t\} \subset \mathbb{R}^p$.
We denote an algorithm ``alg'' with hyperparameters $a, b$ as $\textsc{Alg}(a, b)$.
We write $\mathbf{H}(\mathbf{x})$ for the Hessian at parameter $\mathbf{x}$, and sometimes write $\mathbf{H}_t$ as an abbreviation for $\mathbf{H}(\mathbf{x}_t)$, the Hessian at step $t$.
We denote the largest positive eigenvalue of a matrix $\mathbf{H}$ as $\lambda_1(\mathbf{H})$.
We call the largest positive eigenvalue of a Hessian the \emph{sharpness}, and the largest positive eigenvalue of a preconditioned Hessian the \emph{preconditioned sharpness}.

\paragraph{Gradient descent and momentum}
The most basic optimization algorithm is $\textsc{GD}(\eta)$, defined by the parameter update $\mathbf{x}_{t+1} = \mathbf{x}_t - \eta \, \mathbf{g}_{t+1}$,  where $\mathbf{g}_{t+1}$ is a full-batch or minibatch gradient of $f$ computed at $\mathbf{x}_t$.
Vanilla gradient descent can be accelerated by the use of momentum, in which case the parameter update becomes $\mathbf{x}_{t+1} = \mathbf{x}_t - \eta \, \mathbf{m}_{t+1}$, where $\mathbf{m}_{t+1}$ is a momentum vector.
Momentum comes in two popular flavors: heavy ball (HB) and Nesterov.
Both of these, in turn, can be parameterized in the ``standard'' fashion or the ``EMA'' (exponential moving average) fashion.
The momentum update for $\textsc{StandardHB}(\eta, \beta_1)$ is $\mathbf{m}_{t+1} = \beta_1 \, \mathbf{m}_t + \mathbf{g}_{t+1}$, whereas the update for  $\textsc{EmaHB}(\eta, \beta_1)$ is $\mathbf{m}_{t+1} = \beta_1 \mathbf{m}_t + (1 - \beta_1) \mathbf{g}_{t+1}$.

\begin{wrapfigure}{l}{0.35\textwidth}
  \begin{center}
    \includegraphics[width=0.32\textwidth]{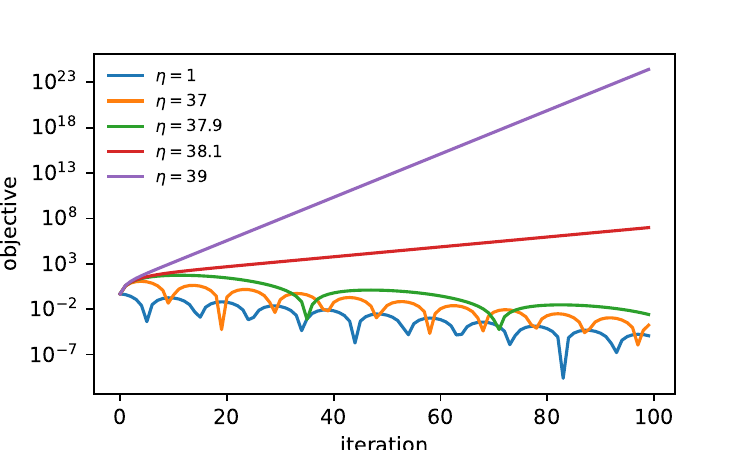}
  \end{center}
  \caption{We optimize the quadratic objective $f(x) = \frac{1}{2} x^2$ using $\textsc{EmaHB}(\eta, 0.9)$ at various $\eta$.  Observe that learning rates $\eta$ above 38 diverge.}
  \label{fig:ema-momentum}
\end{wrapfigure}

Each of these algorithms, when run on any quadratic function (and in particular, when run on the local quadratic Taylor approximation to a neural training objective), evolves independently along each Hessian eigenvector as a linear recurrence relation (Appendix A).
If any Hessian eigenvalue exceeds a certain algorithm-dependent threshold, then the linear recurrence relation for the corresponding eigenvector is \emph{unstable} and diverges exponentially; this implies that the iterate will oscillate with exponentially increasing magnitude along this eigenvector.
The stability thresholds are: $2/\eta$ for $\textsc{GD}(\eta)$, $\frac{2 + 2 \beta_1}{\eta}$ for $\textsc{StandardHB}(\eta, \beta_1)$, and $\frac{2 + 2 \beta_1}{\eta (1 - \beta_1)}$ for $\textsc{EmaHB}(\eta, \beta_1)$.

We illustrate this threshold behavior in Figure \ref{fig:ema-momentum}, where we optimize the one-dimensional quadratic function $f(x) = \frac{1}{2} x^2$ using EMA-style heavy ball momentum with $\beta_1 = 0.9$ and varying learning rates.
Observe that the iterates diverge whenever the learning rate $\eta$ exceeds $\frac{2 + 2 \beta_1}{(1 - \beta_1)} = 38$.

\paragraph{Preconditioned gradient descent}
Gradient descent, or its momentum variants, can be preconditioned using a (static) preconditioner $\mathbf{P}$.
In this case, the update rule becomes $ \mathbf{x}_{t+1} = \mathbf{x}_t - \eta \, \mathbf{P}^{-1} \, \mathbf{m}_{t+1}$, where $\mathbf{m}_t$ is updated according to one of the momentum rules described above.
Often, $\mathbf{P}$ is a diagonal matrix, in which case $\mathbf{P}^{-1}$ can be interpreted as per-parameter learning rates.
A special case of particular interest in this paper is preconditioned gradient descent with EMA-style heavy ball momentum, which we denote as $\textsc{PrecondEmaHB}(\eta, \beta_1, \mathbf{P})$.
On a quadratic objective with Hessian $\mathbf{H}$, preconditioned gradient descent will diverge if the maximum eigenvalue of the \emph{preconditioned Hessian} $\mathbf{P}^{-1/2} \mathbf{H} \mathbf{P}^{-1/2}$ (the \emph{preconditioned sharpness}) exceeds the stability threshold of the non-preconditioned algorithm.
In that event, the component of the iterate that is aligned with $(\mathbf{P}^{1/2})^T \mathbf{v}_1$, where $\mathbf{v}_1$ is the top eigenvector of $\mathbf{P}^{-1/2} \mathbf{H} \mathbf{P}^{-1/2}$, will oscillate with exponentially increasing magnitude, and diverge.
Note that $\mathbf{P}^{-1/2} \mathbf{H} \mathbf{P}^{-1/2}$ shares eigenvalues with the similar matrix $\mathbf{P}^{-1} \mathbf{H}$, so the preconditioned sharpness can be written as both $\lambda_1(\mathbf{P}^{-1/2} \mathbf{H} \mathbf{P}^{-1/2})$ or $\lambda_1(\mathbf{P}^{-1} \mathbf{H})$.

\paragraph{Adaptive gradient methods}
\emph{Adaptive} gradient methods are variants of preconditioned gradient descent in which the preconditioner is changing rather than fixed, i.e. the parameter update is $\mathbf{x}_{t+1} = \mathbf{x}_t - \eta \mathbf{P}_{t+1}^{-1} \mathbf{m}_{t+1}$, where $\mathbf{P}_{t+1}$ is the latest preconditioner.
For example, $\textsc{Rmsprop}(\eta, \beta_2, \epsilon)$ \citep{tieleman2012divide} updates its preconditioner according to:
\begin{align}
    \boldsymbol{\nu}_{t+1} = \beta_2 \boldsymbol{\nu}_t + (1 - \beta_2)\mathbf{g}_{t+1}^{\circ 2} \quad\quad \mathbf{P}_{t+1} = \text{diag}(\boldsymbol{\nu}_{t+1}^{1/2}) + \epsilon \, \mathbf{I}. \label{eq:rmsprop-precondition}
\end{align}
A more popular algorithm is $\textsc{Adam}(\eta, \beta_1, \beta_2, \epsilon)$ \citep{kingma2014adam}, which adds EMA-style heavy ball momentum to rmsprop.
Adam has an optional ``bias correction'' scheme.
With no bias correction, Adam employs the rmsprop preconditioner rule in Eq \eqref{eq:rmsprop-precondition}; with bias correction, the Adam preconditioner is:
\begin{align}
    \mathbf{P}_{t+1} &= \left(1 - \beta^{t+1}_1 \right) \; \left[ \text{diag} \left (\sqrt{\frac{\boldsymbol{\nu}_{t+1}}{1 - \beta_2^{t+1}}}\right) + \epsilon \, \mathbf{I} \right].
    \label{eq:adam-precondition}
\end{align}
In contrast to gradient descent (and preconditioned gradient descent), adaptive gradient methods do not evolve as linear recurrence relations on quadratic functions.
Thus, it is a priori unclear whether their local stability can be modeled using an eigenvalue condition.

\section{Related Work}
\label{sec:related-work}

\paragraph{Training dynamics of non-adaptive gradient descent}
Recent empirical studies \citep{jastrzebski2020break, lewkowycz2020large, cohen2021gradient, lee2022implicit} have shed substantial light on the training dynamics of non-adaptive gradient descent in deep learning.
When training neural networks, gradient descent tends to continually move in a direction that increases the sharpness \citep{jastrzkebski2018relation, jastrzebski2020break, cohen2021gradient}.
This phemonemon was dubbed \emph{progressive sharpening} in \citep{cohen2021gradient}, and remains poorly understood.
In the special case of \emph{full-batch} gradient descent, the sharpness usually rises past the optimizer's stability threshold, at which point (the ``breakeven point'' \citep{jastrzebski2020break}) the optimizer becomes destabilized by exponentially growing movement along the Hessian's top eigenvector.
On quadratic objective functions, this behavior would lead to divergence.
However, neural network training objectives are not quadratic, and gradient descent typically does not diverge; instead, it enters a regime called the Edge of Stability (EoS) \citep{cohen2021gradient} in which the sharpness hovers just above, or oscillates around, the stability threshold.
At the EoS, gradient descent is constantly being repelled from regions of the loss landscape with sharpness exceeding the stability threshold.
Prior work has not discussed preconditioned gradient descent, but in the next section we confirm that the results of \citep{cohen2021gradient} carry over to the non-adaptive preconditioned setting.
It remains unclear how gradient descent is able to safely train at the EoS without diverging \cite{ahn2022understanding, ma2022multiscale}, though \citep{ma2022multiscale} has suggested that ``subquadratic growth'' of the training objective may play a role.
Note that these results only apply to \emph{full-batch} gradient descent.  In the more general case of SGD, similar effects seem to occur \citep{jastrzkebski2018relation, jastrzebski2020break}, as we review in \S \ref{sec:sgd}.

The fact that non-adaptive gradient descent is blocked from entering sharp (as quantified by maximum Hessian eigenvalue) regions of the loss landscape constitutes one implicit bias \citep{neyshabur2014search} of non-adaptive gradient descent.
It is plausible that this implicit bias could impact generalization (e.g. see \citep{mulayoff2021the, jastrzebski2021catastrophic}).

\paragraph{Understanding adaptive gradient methods}
There have been a number of convergence analyses of adaptive gradient methods \citep{chen2018convergence, zhou2018convergence, li2019convergence, ward2019adagrad, xie2020linear, defossez2020simple}, none of which model the behavior that we observe here.
Beyond formal convergence analyses, \citep{barakat2018convergence, da2020general} proposed to model Adam as an system of ordinary differential equations in continuous time; this approach also cannot explain the unstable dynamics that we observe.
\citep{balles2018dissecting} argued that Adam should be viewed as a variant of sign gradient descent.
Most relevant to our paper, \cite{ma2020qualitative} conducted a qualitative study of full-batch Adam during neural network training.
They observed small bumps and large spikes in the training loss curve; our paper explains this phenomenon as arising from dynamical instability.

\paragraph{Implicit bias of adaptive gradient methods}

Absent well-tuned regularization, adaptive gradient methods have been reported to generalize worse than non-adaptive optimizers \citep{wilson2017marginal, keskar2017improving, loshchilov2016sgdr}.
However, it has been far from clear why, in deep learning, adaptive optimizers should even find consistently different solutions than non-adaptive optimizers (let alone why these different solutions should generalize worse).
\citep{wilson2017marginal} constructed a synthetic task where Adam provably generalizes worse than gradient descent.
However, synthetic tasks also exists where the reverse is true \citep{agarwal2020disentangling}.
\citep{wang2021implicit} proved that rmsprop on homogenous neural networks is, like gradient descent \citep{lyu2019gradient}, implicitly biased towards maximum margin solutions.
\citep{granziol2020explaining} argued that adaptive methods take steps that are too large in small-Hessian-eigenvalue directions, which are more likely to be noise than signal.

Our paper gives evidence for a different implicit bias: adaptive gradient methods are liable to find higher-curvature solutions than non-adaptive algorithms, since whereas non-adaptive algorithms are blocked from high-curvature regions, adaptive optimizers can evade this restriction.
\citep{wu2018how} briefly speculated that this might be the case, but ran no experiments with adaptive optimizers (only L-BFGS).
Meanwhile, \citep{granziol2020flatness} empirically observed that a VGG trained using an adaptive optimizer had 40x the sharpness than when trained using SGD, and \citep{chen2021vision} found that MLP-mixers and ViTs trained with Adam converged to far sharper solutions than a ResNet trained with SGD.

\citep{zhou2020towards} and \citep{xie2020adai} previously argued theoretically that SGD can more quickly escape certain ``sharp minima'' than Adam, but using a different definition of sharpness, and a mechanism tied inextricably to stochasticity.
In contrast, we demonstrate that even full-batch Adam converges to higher-curvature solutions than full-batch gradient descent, implying that stochasticity is not crucial.

\section{Full-batch adaptive optimizers train at the Adaptive Edge of Stability}

\begin{figure}[h!]
    \includegraphics[width=14cm]{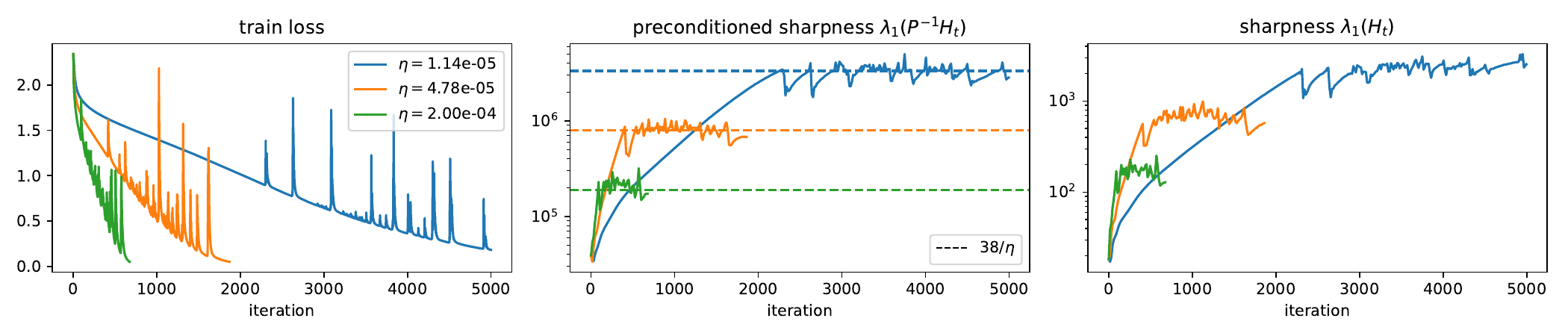}
    \caption{\textbf{``Frozen Adam'' (preconditioned momentum) trains at the Edge of Stability.}  We train a fully-connected network on CIFAR-10 using ``frozen Adam,'' i.e. preconditioned gradient descent with EMA-style heavy ball momentum and a fixed preconditioner $\mathbf{P}$. Consistent with \citep{cohen2021gradient}, the preconditioned sharpness $\lambda_1(\mathbf{P}^{-1} \, \mathbf{H}_t)$ rises until equilibrating at the stability threshold of $38/\eta$.  The raw sharpness  $\lambda_1(\mathbf{H}_t)$ mostly ceases to increase at the EoS.}
    \label{fig:fixed-preconditioner}
\end{figure}

\paragraph{Warm-up: ``frozen Adam''}
Before turning to adaptive preconditioning, we first confirm that prior work \citep{cohen2021gradient} generalizes to the setting of preconditioned gradient descent.
We train a network using full-batch Adam (no bias correction), and extract the second moment accumulator $\boldsymbol{\nu}_t$ at some midpoint step $t$ during training.
We then train the same architecture from scratch at a range of learning rates using full-batch ``frozen Adam,'' i.e. Adam where the second-moment accumulator is held fixed at that particular  $\boldsymbol{\nu}_t$ rather than updated via Eq \eqref{eq:rmsprop-precondition} or Eq \eqref{eq:adam-precondition}.
Note that ``frozen Adam'' is equivalent to $\textsc{PrecondEmaHB}(\eta, \beta_1, \mathbf{P})$ where $\mathbf{P} = \text{diag}(\boldsymbol{\nu}_t^{1/2}) + \epsilon \mathbf{I}$.
Thus frozen Adam is dynamically unstable on the training objective's quadratic Taylor approximation whenever the preconditioned sharpness $\lambda_1(\mathbf{P}^{-1} \, \mathbf{H}_t)$ exceeds the stability threshold for EMA-style heavy ball momentum.
This threshold is $\frac{2 + 2 \beta_1}{(1 - \beta_1) \eta} = \frac{38}{\eta}$ for our choice of $\beta_1 = 0.9$.
Assuming that the findings of \citep{cohen2021gradient} generalize from non-preconditioned to preconditioned gradient descent, one would expect that during full-batch training by frozen Adam, the preconditioned sharpness would rise until plateauing at that threshold value.
In Figure \ref{fig:fixed-preconditioner}(b), we confirm that this occurs.
For each of several learning rates $\eta$, we plot the evolution of the preconditioned sharpness $\lambda_1(\mathbf{P}^{-1} \mathbf{H}_t)$ while visualizing the stability threshold of $38/\eta$ as a horizontal dashed line.
Observe that the preconditioned sharpness rises until plateauing at the dashed line.
We also plot the evolution of the ``raw'' sharpness $\lambda_1(\mathbf{H}_t)$ in Figure \ref{fig:fixed-preconditioner}(c).
For this network, the raw sharpness largely ceases to increase once training enters the EoS.

\paragraph{Adaptive preconditioning: Adam}
Adam is a more complex dynamical system than ``frozen Adam,'' because its preconditioner evolves in respose to recent gradients; consequently, its dynamics on the local quadratic Taylor approximation does not reduce to a linear recurrence.
Throughout physics and engineering, a time-tested way to understand complex systems is to approximate them by simpler ones that still capture some behavior of interest.
Consider $\text{Adam}$ at any particular step $t_0$, and let $\mathbf{P}_{t_0}$ be Adam's preconditioner at that step.
Over short time horizons $t > t_0$, one can approximate Adam by $\textsc{PrecondEmaHB}(\eta, \beta_1, \mathbf{P}_{t_0})$, i.e. ``frozen Adam'' where the preconditioner is frozen at $\mathbf{P}_{t_0}$.
On the local quadratic Taylor approximation, this simpler algorithm is unstable if the preconditioned sharpness $\lambda_1(\mathbf{P}_{t_0}^{-1} \mathbf{H}_{t_0})$ exceeds the stability threshold of heavy-ball EMA momentum, which is $\frac{2 + 2 \beta_1}{(1 - \beta_1) \eta}$.
Of course, whether this approximation is useful is an empirical question.
And empirically, we now show that Adam has similar stability properties as its frozen counterpart.

In Figure \ref{fig:adam}, we train a fully-connected network on the full CIFAR-10 dataset using full-batch Adam (no bias correction).
In Figure \ref{fig:adam}(a), we vary the learning rate $\eta$ while fixing $\beta_2 = 0.999$ and $\beta_1 = 0.9$.
Observe that for every learning rate $\eta$, the preconditioned sharpness $\lambda_1(\mathbf{P}_t^{-1} \mathbf{H}_t)$ equilibrates right around the stability threshold of frozen Adam.
In Figure \ref{fig:adam}(b), we vary $\beta_1$ as we fix $\eta = 0.0001$ and $\beta_2 = 0.999$.
Observe that the same phenomenon occurs.
Finally, in Figure \ref{fig:adam}(c), we vary $\beta_2 \in \{0.9, 0.99, 0.999\}$ as we fix $\beta_1 = 0.9$ and $\eta \in \{$5e-5, 4e-4$\}$.
The same phenomenon occurs, even when $\beta_2$ takes the relatively small value of 0.9, which is when one might expect Adam's behavior to differ the most from frozen Adam's, because the preconditioner adapts faster.
Thus, we see that Adam becomes unstable whenever its frozen counterpart would become unstable.
Nevertheless, in \S \ref{sec:explain}, we will see that Adam's behavior at the AEoS differs in other respects from that of frozen Adam at the EoS.

\paragraph{Other architectures}
In Figure \ref{fig:architectures}, we verify that the full-batch phenomenon generalizes to other computer vision architectures.
In particular, we consider (a) a CNN on CIFAR-10; (b) an un-normalized Wide ResNet (WRN) \citep{zagoruyko2016wide} on CIFAR-10; and (c) a (batch-normalized) WRN on CIFAR-100.
Furthermore, Figure \ref{fig:wmt_bs_comp} in the subsequent section verifies that our \emph{minibatch} findings apply to transformers on WMT machine translation.

\begin{figure}[h!]
    \includegraphics[width=14cm]{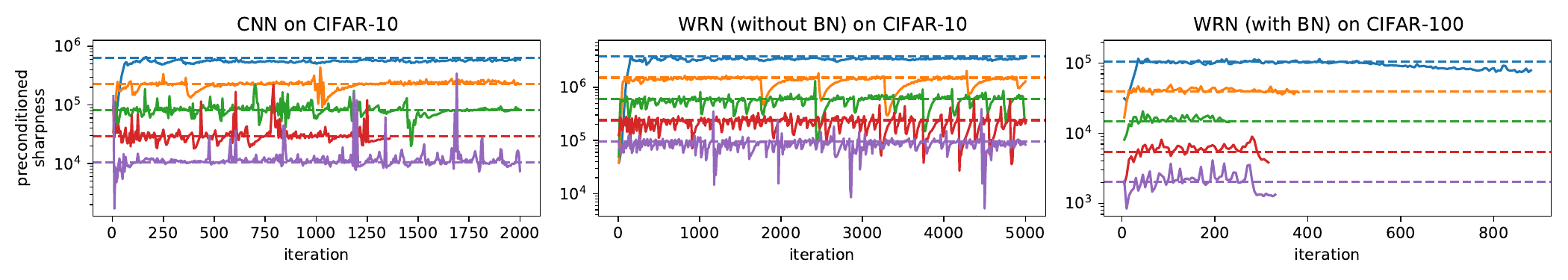}
    \caption{\textbf{The phenomenon generalizes to other architectures}.  We train three vision architectures using full-batch Adam with $\beta_1 = 0.9$ and $\beta_2= 0.999$ at a range of learning rates (colors).  In each case, the preconditioned sharpness equilibrates at the threshold $38/\eta$.  Each network was trained until either reaching a milestone train loss value, or until reaching a step limit.}
    \label{fig:architectures}
\end{figure}

\paragraph{Other full-batch adaptive gradient methods}
In Figure \ref{fig:algorithms}, we train the same fully-connected network on CIFAR-10 using eight adaptive gradient algorithms in the full-batch regime: (a) Adam with bias correction, which essentially decreases the preconditioner on a schedule in the early stage of training; (b) AdamW \citep{loshchilov2017decoupled}, which employs decoupled weight decay; (c) Adafactor \citep{shazeer2018adafactor}, which maintains a factored layer-wise preconditioner; (d) Amsgrad \citep{reddi2018convergence}, which ensures that the preconditioner is entrywise non-decreasing over time; (e) Padam \citep{chen2018closing}, which employs an exponent smaller than $\frac{1}{2}$ (we use $0.25$) in the definition of the preconditioner Eq \eqref{eq:rmsprop-precondition}; (f) Nadam \citep{Dozat2016IncorporatingNM}, which uses Nesterov rather than heavy ball momentum; (g) rmsprop \citep{tieleman2012divide}, which lacks momentum; and finally (h) Adagrad \citep{duchi11adaptive}, which preconditions using the sum of past squared gradients, rather than an exponential moving average.
In all cases, we observe that the preconditioned sharpness $\lambda_1(\mathbf{P}_t^{-1} \mathbf{H}_t)$ equilibrates around the appropriate stability threshold (written in parentheses).

\begin{figure}[h!]
    \includegraphics[width=14cm]{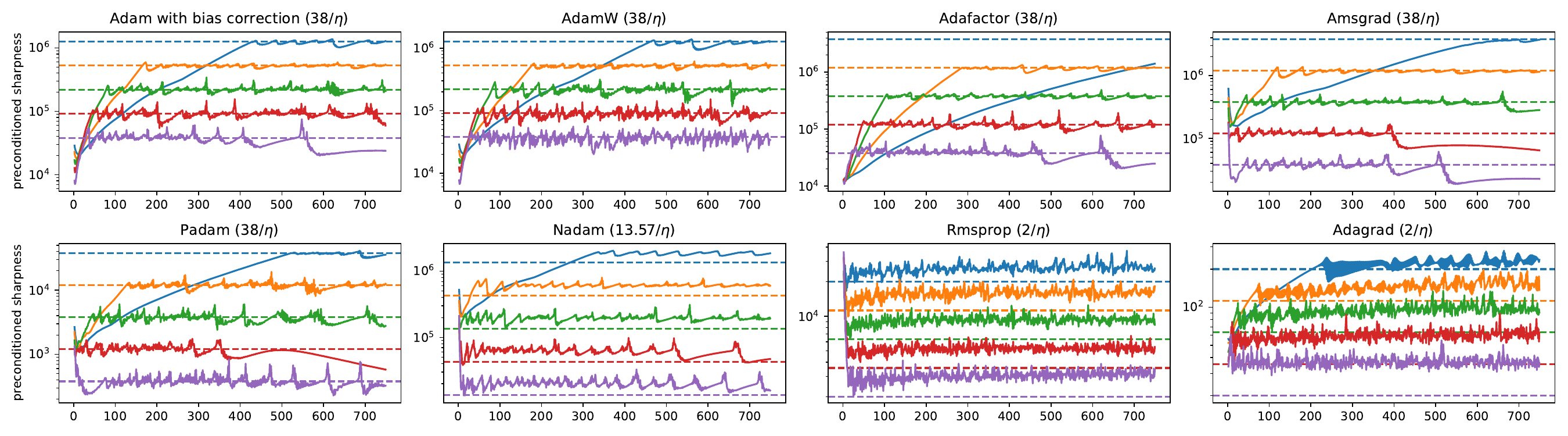}
    \caption{\textbf{Other adaptive gradient algorithms train at the AEoS}.  We train a FC network on CIFAR-10 using eight adaptive optimizers in full-batch mode.  We train each algorithm at five learning rates (colors). In each case, the preconditioned sharpness equilibrates at, or just above, the stability threshold (written in parentheses). Note that the qualitative behavior of the preconditioned sharpness depends on the presence and type of momentum; see Appendix C.}
    \label{fig:algorithms}
\end{figure}

\paragraph{A corner case}
When running full-batch Adam at extremely small learning rates, we sometimes observe that the preconditioned sharpness equilibrates \emph{short of} the stability threshold.  
We elaborate on this corner case, and offer an explanation, in Appendix D.

\section{The minibatch setting}
\label{sec:sgd}

We now move beyond the full-batch setting to the more general setting of minibatch training.  In the case of gradient descent and minibatch SGD, it is clear from prior work \citep{jastrzkebski2018relation, jastrzebski2020break} that during minibatch training, the sharpness is subject to similar effects as during full-batch training.
For one, provided that training is successful, the sharpness never ventures more than a bit above the stability threshold of the corresponding full-batch algorithm \citep{jastrzebski2020break}.  This can be explained by the fact that SGD is unstable in expectation whenever GD is unstable \citep{giladi2019stability}.  Yet additional factors seem to also be at play in the case of SGD.  Specifically, it has been observed that at small batch sizes (where there is more gradient noise), the mid-training sharpness is smaller.  One hypothesized explanation \citep{wu2018how, jastrzebski2020break} is that in the presence of gradient noise, SGD becomes unstable at lower sharpnesses. 
\begin{figure}[h!]
    \includegraphics[width=14cm]{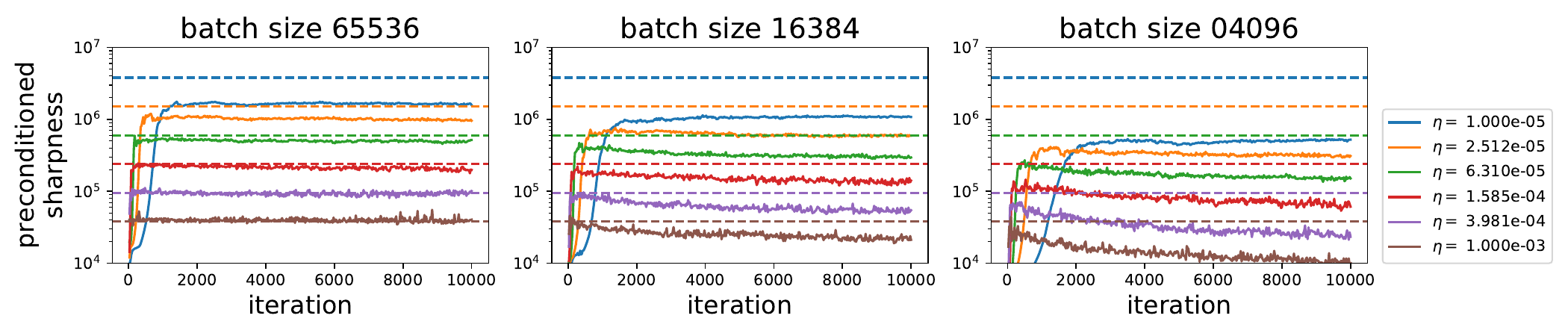}
    \caption{\textbf{During minibatch Adam, the preconditioned sharpness behaves analogous to the sharpness during minibatch SGD}.  We train a Resnet-50 on ImageNet using minibatch Adam.  The preconditioned sharpness is (1) below the full-batch stability threshold (pictured as a horizontal line), (2) smaller when the batch size is smaller, and (3) smaller when the learning rate is larger. }
    \label{fig:sgd}
\end{figure}

We now demonstrate that the behavior of the preconditioned sharpness during minibatch Adam parallels that of the sharpness during minibatch SGD.  Namely, we observe that during minibatch Adam, the preconditioned sharpness (1) never rises more than a bit beyond the stability threshold of the full-batch algorithm, provided that training succeeds; (2) tends to be smaller, mid-training, when the learning rate is large; and (3) tends to be smaller, mid-training, when the batch size is small. 

In Figure \ref{fig:sgd}, we train a Resnet-50 on ImageNet at three batch sizes.
Observe that for the largest batch size considered, the preconditioned sharpness behaves similar to the full-batch setting (almost rising to the full-batch stability threshold); yet, as the batch sizes decreases, the trajectory of the preconditioned sharpness shifts downwards.

In Figure \ref{fig:wmt_bs_comp}, we train both a pre-LayerNorm (Pre-LN) and a post-LayerNorm (Post-LN) Transformer on the En->De WMT translation task.
Following standard practice \citep{popel2018training}, we train using a linear warmup schedule lasting 40k steps.
We train at both a smaller batch size of 1024 and a larger batch size of 10240.
The dotted line depicts the stability threshold of $38/\eta$; this threshold is constantly decreasing as the learning rate is warmed up.
Observe that the preconditioned sharpness initially rises until reaching either the stability threshold (larger batch size), or a value just short of the stability threshold (smaller batch size).
Then, as learning rate warmup continues, the preconditioned sharpness decreases to track the stability threshold.
In effect, the preconditioned sharpness is being ``pushed down'' by Adam's learning rate warmup.
Note that \citep{gilmer2021loss} previously showed that a warmup schedule for momentum SGD likewise ``pushed down'' the (unpreconditioned) sharpness in Transformer training.
The Post-LN Transformer is known to be less stable than the Pre-LN variant \cite{xiong2020layer, liu2020understanding}; this instability can be seen in the spike in $\lambda_1(\mathbf{P}^{-1} \mathbf{H})$ (and training loss) towards the end of the warmup period.

\begin{figure}[h!]
    \includegraphics[width=14cm]{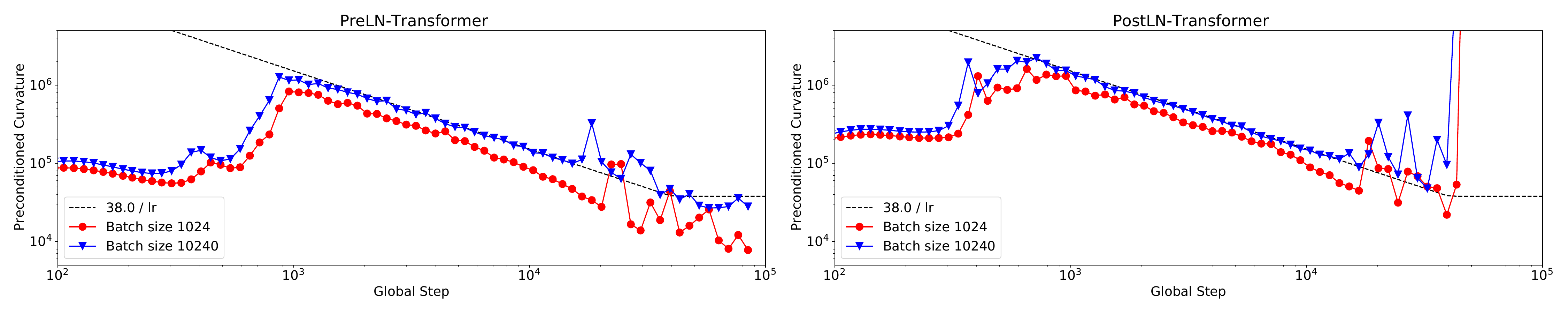}
    \caption{{\bf Learning rate warmup gradually reduces the preconditioned sharpness during optimization.} {\em Left:} The evolution of the preconditioned sharpness for a 6L-6L encoder-decoder Pre-LN transformer trained on the WMT En=>De task at $\eta = .001$ with a linear warmup period of 40000 steps. {\em Right:}  Same as left, but with a PostLN Transformer. In both cases the preconditioned curvature closely tracks the $38 / \eta$ bound during warmup, however there is a noticeable gap at the smaller batch size. The PostLN Transformer training fails late in the warmmup period. } 
    \label{fig:wmt_bs_comp}
\end{figure}

\section{A closer look at the AEoS}
\label{sec:explain}

In this section, we take a closer look at Adam's behavior at the AEoS.
We will see that this behavior sometimes differs substantially from that of non-adaptive optimizers.
In particular, whereas non-adaptive optimizers at the EoS are blocked from entering high-curvature regions of parameter space, adaptive gradient methods at the AEoS can and do enter high-curvature regions via their ability to adapt the preconditioner.

\begin{figure}[h!]
    \includegraphics[width=14cm]{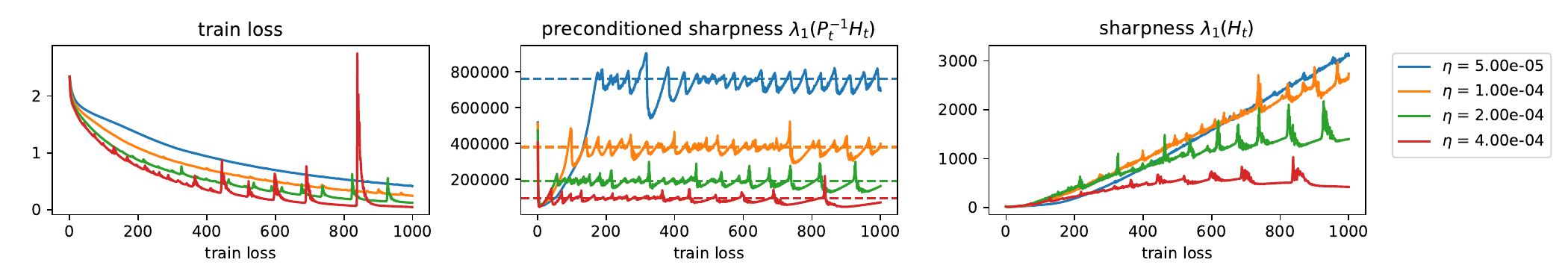}
    \caption{\textbf{The sharpness continues to rise even as Adam trains at the AEoS.}  We train a fully-connected network on CIFAR-10 using full-batch Adam with $\beta_1 = 0.9$, $\beta_2 = 0.99$, and four learning rates.   Observe that while the \emph{preconditioned} sharpness flatlines, the \emph{raw} sharpness continues to rise.}
    \label{fig:hessian-rises}
\end{figure}

In Figure \ref{fig:hessian-rises}, we train a fully-connected network on CIFAR-10 using full-batch Adam at four learning rates.
Observe that even though the preconditioned sharpness $\lambda_1(\mathbf{P}^{-1}_t \mathbf{H}_t)$ flatlines, the raw sharpness $\lambda_1(\mathbf{H}_t)$ continues to rise while the network is training at the AEoS.
This behavior stands in stark contrast to the behavior of frozen Adam at the EoS in Figure \ref{fig:fixed-preconditioner}; there, the raw sharpness ceased to appreciably increase after training had entered the EoS.

Why might this be the case?
A crucial distinction between adaptive and non-adaptive optimizers is that the latter are necessarily repelled from high-curvature regions of parameter space.
For example, $\textsc{EmaHB}(\eta, \beta_1)$ is repelled from the set $\{\mathbf{x}: \lambda_1(\mathbf{H}(\mathbf{x})) > \frac{2+2\beta_1}{\eta(1 - \beta_1)} \}$, and $\textsc{PreconEmaHB}(\eta, \beta_1, \mathbf{P})$ is repelled from the set  $\{\mathbf{x}: \lambda_1( \mathbf{P}^{-1} \mathbf{H}(\mathbf{x})) > \frac{2+2\beta_1}{\eta(1 - \beta_1)} \}$.
By contrast, while Adam's trajectory $\{(\mathbf{P}_t, \mathbf{x}_t)\}_t$ must approximately respect the constraint $\lambda_1(\mathbf{P}_t^{-1} \mathbf{H}(\mathbf{x}_t)) \le \frac{2 + 2 \beta_1}{\eta (1 - \beta_1)}$, Adam is allowed to change the preconditioner, and hence can train stably in regions of arbitrarily high curvature.

To gain more insight into the behavior of Adam at the AEoS, in Figure \ref{fig:zoomin} we ``zoom in'' on 200 steps of Adam training, for both a small learning rate (5e-5, the blue line in Figure \ref{fig:hessian-rises}) and a large learning rate (4e-4, the red line in Figure \ref{fig:hessian-rises}). 
In the top row, which corresponds to the small learning rate, we observe the following cycle.
Whenever the preconditioned sharpness is below the stability threshold, both the sharpness and preconditioned sharpness rise.
Eventually, the preconditioned sharpness crosses the stability threshold --- an event which we mark with a vertical dotted red line.
This causes Adam to oscillate explosively along a certain direction, as discussed in \S \ref{sec:background} (see Appendix E for more evidence).
The large gradients associated with this explosive growth increase the second-moment accumulator $\boldsymbol{\nu}_t$ and hence Adam's preconditioner $\mathbf{P}_t$.
This increase in $\mathbf{P}_t$ in turn causes the preconditioned sharpness $\lambda_1(\mathbf{P}_t^{-1} \mathbf{H}_t)$ to decrease until it drops below the stability threshold, at which point the large gradients go away.
The cycle then repeats itself. 
By this mechanism, the preconditioned sharpness equilibrates at the stability threshold, but the raw sharpness continues to grow.
In other words, instability generates large gradients, which in turn increase the preconditioner, allowing Adam to continue advancing into regions of higher sharpness.

\begin{figure}[h!]
    \includegraphics[width=14cm]{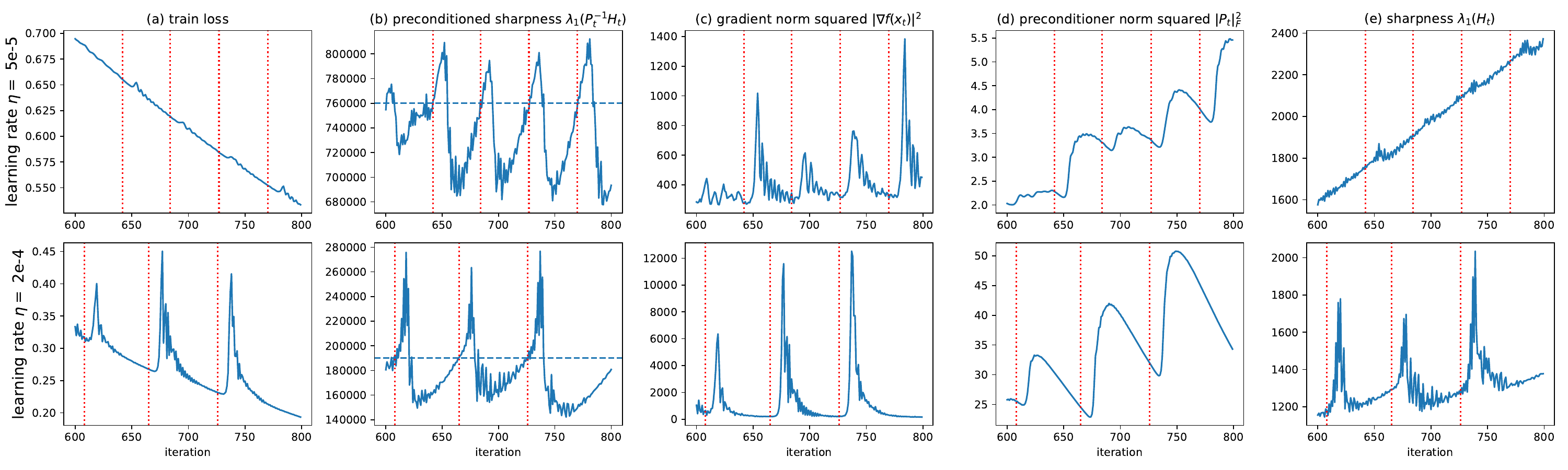}
    \caption{\textbf{Zoom in on 200 steps of training.}  For two learning rates from Figure 9 (top row = 5e-5, bottom row = 2e-4), we zoom in on 200 steps of training, plotting five important quantities.  The red vertical lines mark the moments at which the preconditioned sharpness crosses the stability threshold. }
    \label{fig:zoomin}
\end{figure}

The bottom row of Figure \ref{fig:zoomin} depicts 200 steps of Adam at a larger learning rate.
At this larger learning rate, Adam behaves somewhat differently.
As above, we see that whenever the preconditioned sharpness is below the threshold, both the sharpness and preconditioned sharpness rise.
And, as above, we see that after the preconditioned sharpness crosses the threshold (red vertical lines), the gradient norm begins to increase explosively.
However at the larger learning rate, this instability is ``resolved'' differently than at the small learning rate.
In particular, for this large learning rate, we observe that the instability is resolved when both the sharpness and the preconditioned sharpness drop, indicating that the optimizer has escaped to a region where the curvature in many different directions (not just the maximum eigenvector of the preconditioned Hessian) is lower.
This behavior is similar to the ``catapult'' effect studied in \citep{lewkowycz2020large} for non-adaptive optimizers.

\paragraph{The role of Adam's hyperparameters}

The preceding discussion suggests that Adam's behavior at the AEoS --- and in particular, the extent to which the stability-based constraint on $\lambda_1(\mathbf{P}^{-1}_t \mathbf{H}_t)$ is maintained by adapting $\mathbf{P}$ vs. by constraining $\mathbf{H}$ --- may depend on Adam's hyperparameters.
To test this hypothesis, in Figure \ref{fig:hyperparameters}(a), we use full-batch Adam at a range of learning rates to train a Wide Resnet on CIFAR-100 until reaching a training loss value of 0.1.
We plot the sharpness $\lambda_1(\mathbf{H}(\mathbf{x}^*))$ at the solution $\mathbf{x}^*$, as a function of the learning rate.
Observe that full-batch Adam at small learning rates finds sharper solutions.
In Figure \ref{fig:hyperparameters}(b), we confirm that this trend extends to minibatch Adam at the large batch size of 4096.
Meanwhile, in Figure \ref{fig:hyperparameters}(c), we sweep full-batch Adam at a range of values of $\beta_2$.
We observe that when $\beta_2$ is small, Adam tends to find sharper solutions.
Figure \ref{fig:hyperparameters}(d) demonstrates that this trend extends to minibatch Adam at batch size 4096.
Incidentally, we show in Appendix F that the models trained with higher $\eta$ or higher $\beta_2$ generalize better.

\begin{figure}[h!]
    \includegraphics[width=14cm]{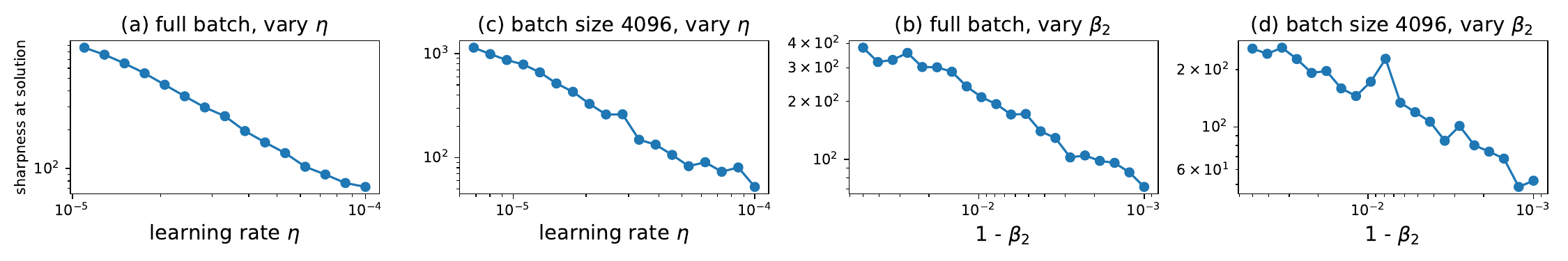}
    \caption{\textbf{Adam finds sharper solutions when $\eta$ or $\beta_2$ is small.}  We train a WRN on CIFAR-100 until reaching train loss 0.1, and we record the sharpness $\lambda_1(\mathbf{H}(\mathbf{x}^*))$ at the solution $\mathbf{x}^*$.
    In (a, b) we sweep the learning rate $\eta$, while in (c, d) we sweep $\beta_2$.
    In (a, c) we train in full-batch mode, while in (b, d) we use the large batch size 4096.}
    \label{fig:hyperparameters}
\end{figure}

One speculative hypothesis is that Adam's ability to resolve instabilities via preconditioner adaptation is contingent on how fast the preconditioner adapts (smaller $\beta_2$ = faster adaptation) relative to the speed at which the optimizer is moving (larger $\eta$ = faster movement).

\section{Adam finds sharper solutions than momentum}

Adam and other adaptive gradient methods are able to train stably in arbitrarily sharp regions of parameter space, whereas momentum gradient descent is repelled from regions of parameter space where the sharpness exceeds the algorithm's stability threshold.
We now demonstrate that as a consequence, Adam sometimes converges to far sharper solutions than momentum GD.
In Figure \ref{fig:adam-sharper-momentum}, we train a WRN on CIFAR-100 (left two panes) and a fully-connected network on CIFAR-10 (right two panes) using both Adam and momentum GD.
It is impossible to make a blanket claim that Adam finds sharper solutions than momentum GD, since momentum at a small learning rate finds sharper solutions than Adam at a large one; however, we observe that \emph{for a given training speed}, Adam converges to a sharper solution.

\begin{figure}[h!]
    \includegraphics[width=14cm]{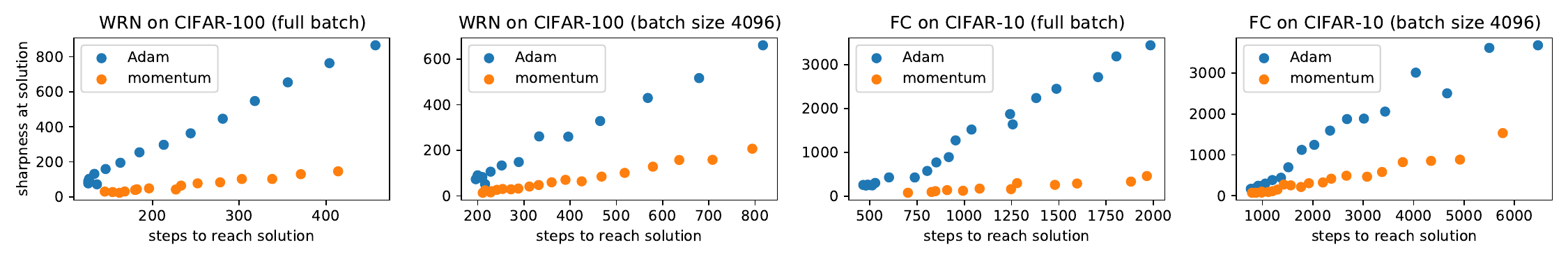}
    \caption{\textbf{At comparable training speeds, Adam finds sharper solutions than momentum gradient descent.}  We train networks using both Adam (blue) and momentum (orange) at a range of learning rates.  For each algorithm and learning rate, we plot the maximum Hessian eigenvalue at the solution as a function of the number of steps to reach the solution.  For both algorithms, smaller learning rates (= longer training speeds) find sharper solutions; but in general, Adam finds a much sharper solution than momentum GD for the same training time.}
    \label{fig:adam-sharper-momentum}
\end{figure}

\section{Conclusion}

In a recent year, Adam \citep{kingma2014adam} was the most highly cited scientific paper across all fields of science \citep{crew_2020}.
Yet, very little is known concretely about the training dynamics of Adam or other adaptive gradient methods in deep learning.
Since the training dynamics affect both the optimization and generalization performance of neural networks in mysterious ways, any basic knowledge about the training dynamics could pay dividends for future research.
This paper has demonstrated that in the special setting of full-batch training, there is an equilibrium rule that is typically maintained throughout training.
We have provided evidence that a similar equilibrium may exist in the more practical minibatch setting, and we hope that an analogous equilibrium rule will eventually be found.

\section{Acknowledgements}
The authors are

\clearpage
\bibliographystyle{plain} 

\appendix

\newpage

\section{Mathematical background}

\subsection{Stability of gradient descent algorithms on quadratics}

We now review the stability properties of gradient descent algorithms on quadratic objective functions.
We first refer the reader to Appendix A of \citep{cohen2021gradient}; that appendix derives the stability thresholds for $\textsc{GD}(\eta)$, $\textsc{StandardHB}(\eta, \beta_1)$, and  $\textsc{StandardNesterov}(\eta, \beta_1)$.
In this appendix, we derive the stability thresholds for $\textsc{EmaHB}(\eta, \beta_1)$, and  $\textsc{EmaNesterov}(\eta, \beta_1)$.
These derivations are minor variants of the corresponding derivations in \citep{cohen2021gradient}, since the EMA-style momentum algorithms yield the same recurrence relation as the standard-style momentum algorithms, except with a learning rate scaled by $(1 - \beta_1)$.

\begin{table}[h]
\begin{center}
\begin{tabular}{ |c|c|c|c| } 
 \hline
 \textbf{Algorithm} & \textbf{Update rule} & \textbf{Stability threshold} & \textbf{Reference for threshold} \\ 
 \hline
 $\textsc{GD}(\eta)$ & 
 $\mathbf{x}_{t+1} = \mathbf{x}_t - \eta \mathbf{g}_{t+1} $
 & $2/\eta$ & \makecell{common knowledge, \\ reproduced in \citep{cohen2021gradient}\\ for completeness} \\ 
 \hline
  $\textsc{StandardHB}(\eta, \beta_1)$ & 
 \makecell{ $\mathbf{m}_{t+1} = \beta_1 \mathbf{m}_t + \mathbf{g}_{t+1}$ \\ $\mathbf{x}_{t+1} = \mathbf{x}_t - \eta \mathbf{m}_{t+1}$}
 & $\frac{2 + 2 \beta_1}{\eta}$ & \makecell{\citep{goh2017why}, reproduced in \citep{cohen2021gradient} \\ for completeness} \\ 
 \hline
  $\textsc{EmaHB}(\eta, \beta_1)$ & 
 \makecell{$\mathbf{m}_{t+1} = \beta_1 \mathbf{m}_t + (1 - \beta_1) \mathbf{g}_{t+1}$ \\ $\mathbf{x}_{t+1} = \mathbf{x}_t - \eta \mathbf{m}_{t+1}$}
 & $\frac{2 + 2 \beta_1}{\eta (1 - \beta_1)}$ & this paper \\ 
 \hline
  $\textsc{StandardNesterov}(\eta, \beta_1)$ & 
 \makecell{$\mathbf{m}_{t+1} = \beta_1 \mathbf{m}_t + \mathbf{g}_{t+1}$ \\
 $\hat{\mathbf{m}}_{t+1} = \beta_1 \mathbf{m}_{t+1} + \mathbf{g}_{t+1}$ \\
 $\mathbf{x}_{t+1} = \mathbf{x}_t - \eta \hat{\mathbf{m}}_{t+1} $}
 & $\frac{2 + 2 \beta_1}{\eta (1 + 2\beta_1)}$ &  \citep{cohen2021gradient} \\
 \hline
  $\textsc{EmaNesterov}(\eta, \beta_1)$ & 
 \makecell{$\mathbf{m}_{t+1} = \beta_1 \mathbf{m}_t + (1 - \beta_1) \mathbf{g}_{t+1}$ \\
 $\hat{\mathbf{m}}_{t+1} = \beta_1 \mathbf{m}_{t+1} + (1 - \beta_1) \mathbf{g}_{t+1}$ \\
 $\mathbf{x}_{t+1} = \mathbf{x}_t - \eta \hat{\mathbf{m}}_{t+1} $}
 & $\frac{2 + 2 \beta_1}{\eta (1 - \beta_1) (1 + 2\beta_1)}$ & this paper \\
 \hline
\end{tabular}
\end{center}
\caption{Gradient descent algorithms and their stability thresholds.}
\label{table:algorithms}
\end{table}

\begin{lemma}
    Consider running $\textsc{EmaNesterov}(\eta, \beta_1)$ on a quadratic objective $f(\mathbf{x}) = \frac{1}{2} \mathbf{x}^T \mathbf{A} \mathbf{x} + \mathbf{b}^T \mathbf{x} + c $ starting from any initialization.  Let $(\mathbf{q}, a)$ be an eigenvector/eigenvalue pair of $\mathbf{A}$.  If $a > \frac{1}{\eta} \left( \frac{2 + 2 \beta_1}{(1 - \beta_1)(1 + 2 \beta_1)} \right)$, then the sequence $\{ \mathbf{q}^T \mathbf{x} \}$ will diverge.
\end{lemma}
\begin{proof}
    First, we re-write $\textsc{EmaNesterov}(\eta, \beta_1)$ as a recursion in $\mathbf{x}_t$ alone:
    $$ \mathbf{x}_{t+1} = (1 + \beta_1) \mathbf{x}_t - \beta_1 \mathbf{x}_{t-1} - \eta(1 - \beta_1)(1 + \beta_1) \nabla f(\mathbf{x}_t) + \eta (1 - \beta_1) \nabla f(\mathbf{x}_{t-1}) $$
    For the quadratic objective, we have $\nabla f(\mathbf{x}_t) = \mathbf{A} \mathbf{x}_t + \mathbf{b}$, so this update becomes:
    $$ \mathbf{x}_{t+1} = (1 + \beta_1)(\mathbf{I} - \eta \mathbf{A}) \mathbf{x}_t - \beta_1 (\mathbf{I} - \eta (1 - \beta_1) \mathbf{A}) \mathbf{x}_{t-1} - \eta (1 - \beta_1) \mathbf{b}.  $$
    This is exactly the recurrence in Theorem 1 of \citep{cohen2021gradient}, except with $(1 - \beta_1) \eta$ in place of $\eta$.  Thus, by the same logic as in Theorem 1 of \citep{cohen2021gradient}, $\mathbf{q}^T \mathbf{x}_t$ diverges if $a > \frac{1}{\eta} \left( \frac{2 + 2 \beta_1}{(1 - \beta_1)(1 + 2 \beta_1)}\right)$.
\end{proof}

\begin{lemma}
    Consider running $\textsc{EmaHB}(\eta, \beta_1)$ on a quadratic objective $f(\mathbf{x}) = \frac{1}{2} \mathbf{x}^T \mathbf{A} \mathbf{x} + \mathbf{b}^T \mathbf{x} + c $ starting from any initialization.  Let $(\mathbf{q}, a)$ be an eigenvector/eigenvalue pair of $\mathbf{A}$.  If $a > \frac{1}{\eta} \left( \frac{2 + 2 \beta_1}{1 - \beta_1} \right)$, then the sequence $\{ \mathbf{q}^T \mathbf{x} \}$ will diverge.
\end{lemma}
\begin{proof}
    First, we re-write $\textsc{EmaHB}(\eta, \beta_1)$ as a recursion in $\mathbf{x}_t$ alone:
    $$ \mathbf{x}_{t+1} = (1 + \beta_1) \mathbf{x}_t - \beta_1 \mathbf{x}_{t-1} - \eta (1 - \beta_1) \nabla f(\mathbf{x}_t). $$
    For the quadratic objective, we have $\nabla f(\mathbf{x}_t) = \mathbf{A} \mathbf{x}_t + \mathbf{b}$, so this update becomes:
    $$ \mathbf{x}_{t+1} = \left[ (1 + \beta_1) \mathbf{I} - \eta (1 - \beta_1) \mathbf{A} \right] \mathbf{x}_t - \beta_1 \mathbf{x}_{t-1} - \eta (1 - \beta_1) \mathbf{b}. $$
    This is exactly the recurrence in Theorem 2 of \citep{cohen2021gradient}, except with $(1 - \beta_1) \eta$ in place of $\eta$.  Thus, by the same logic as in Theorem 2 of \citep{cohen2021gradient}, $\mathbf{q}^T \mathbf{x}_t$ diverges if $a > \frac{1}{\eta} \left( \frac{2 + 2 \beta_1}{1 - \beta_1}\right)$.
\end{proof}

\subsection{Stability of preconditioned gradient descent}

As discussed in \S \ref{sec:background}, gradient descent can be preconditioned using a preconditioner $\mathbf{P}$.
For example, the preconditioned version of $\textsc{EmaHB}(\eta, \beta_1)$, which we denote $\textsc{PreconEmaHB}(\eta, \beta_1, \mathbf{P})$ is:
\begin{align*}
 \mathbf{m}_{t+1} &= \beta_1 \mathbf{m}_t + (1 - \beta_1) \mathbf{g}_{t+1}, \\
 \mathbf{x}_{t+1} &= \mathbf{x}_t - \eta \, \mathbf{P}^{-1} \, \mathbf{m}_{t+1}.
\end{align*}

On quadratic objective functions, the stability behavior of preconditioned gradient descent algorithms differs slightly from that of their non-preconditioned counterparts.
Whereas non-preconditioned gradient descent algorithms become unstable when the maximum eigenvalue of the Hessian $\mathbf{H}$ exceeds a certain stability threshold, preconditioned gradient descent algorithms become unstable when the maximum eigenvalue of the \emph{preconditioned Hessian} $\mathbf{P}^{-1/2} \mathbf{H} \mathbf{P}^{-1/2}$ exceeds the threshold of their non-precondintioned counterpart.

The simplest way to see this is to note that preconditioned gradient descent algorithms are isomorphic to their non-preconditioned counterparts up to a certain reparameterization of the training objective.
In particular, preconditioned gradient descent on an objective $f$ using preconditioner $\mathbf{P}$ is isomorphic to non-preconditioned gradient descent on the reparameterized objective $\hat{f}(\mathbf{x}) = f(\mathbf{P}^{-1/2} \mathbf{x})$, in the sense that the iterates of the two algorithms are related by a linear map.
Below, we make this rigorous for $\textsc{EmaHB}(\eta, \beta_1)$ and $\textsc{PreconEmaHB}(\eta, \beta_1, \mathbf{P})$, but note that one can carry out the same argument for the other momentum algorithms and their preconditioned counterparts:

\begin{prop}
Let $\{(\mathbf{x}_t, \mathbf{m}_t)\}$ denote the iterates of $\textsc{PreconEmaHB}(\eta, \beta_1, \mathbf{P})$ on the objective function $f(\mathbf{x})$, and let $\{(\tilde{\mathbf{x}}_t, \tilde{\mathbf{m}}_t) \}$ denote the iterates of $\textsc{EmaHB}(\eta, \beta_1)$ on the reparameterized objective function $\tilde{f}(\mathbf{x}) = f(\mathbf{P}^{-1/2} \mathbf{x})$ starting from the initialization $(\tilde{\mathbf{x}}_0, \tilde{\mathbf{m}_0}) = (\mathbf{P}^{1/2} \mathbf{x}_0, \mathbf{P}^{-1/2} \mathbf{m}_0)$.
Then we have the equivalence $\tilde{\mathbf{x}}_t = \mathbf{P}^{1/2} \mathbf{x}_t$ and $\tilde{\mathbf{m}}_{t} = \mathbf{P}^{-1/2} \mathbf{m}_t$ for all steps $t$.
\label{prop:equivalence}
\end{prop}
\begin{proof}
    The equivalence is true at initialization, by definition.  Thus, it remains to show that if the equivalence is true at step $t$, it is true at step $t+1$.
    
    The updates for $\{(\mathbf{x}_{t+1}, \mathbf{m}_{t+1})\}$ are:
    \begin{align*}
        \mathbf{m}_{t+1} &= \beta_1 \mathbf{m}_t + (1 - \beta_1) \nabla f(\mathbf{x}_t), \\
        \mathbf{x}_{t+1} &= \mathbf{x}_t - \eta \mathbf{P}^{-1} \mathbf{m}_{t+1}.
    \end{align*}
    Meanwhile, the updates for  $\{(\tilde{\mathbf{x}}_{t+1}, \tilde{\mathbf{m}}_{t+1})\}$ are:
    \begin{align*}
        \tilde{\mathbf{m}}_{t+1} &= \beta_1 \tilde{\mathbf{m}}_t + (1 - \beta_1) \nabla \tilde{f}(\tilde{\mathbf{x}}_t) \\
            &= \beta_1 \tilde{\mathbf{m}}_t + (1 - \beta_1) \mathbf{P}^{-1/2} \nabla f( \mathbf{P}^{-1/2} \tilde{\mathbf{x}}_t),  \\
        \tilde{\mathbf{x}}_{t+1} &= \tilde{\mathbf{x}}_t - \eta \tilde{\mathbf{m}}_{t+1}.
    \end{align*}
    To verify the equivalence $\tilde{\mathbf{m}}_{t+1} = \mathbf{P}^{-1/2} \mathbf{m}_{t+1}$, observe that
    \begin{align*}
        \mathbf{P}^{-1/2} \, \mathbf{m}_{t+1} &= \beta_1 \, \mathbf{P}^{-1/2} \mathbf{m}_t + (1 - \beta_1) \mathbf{P}^{-1/2} \nabla f(\mathbf{x}_t) \\
        &= \beta_1 \tilde{\mathbf{m}_t} + (1 - \beta_1) \mathbf{P}^{-1/2} \nabla f(\mathbf{P}^{-1/2} \tilde{\mathbf{x}}_t) \\
        &= \tilde{\mathbf{m}}_{t+1}.
    \end{align*}
    where the two replacements in the middle step are due to the inductive hypothesis holding for step $t$.
    
    To verify the equivalence $\tilde{\mathbf{x}}_{t+1} = \mathbf{P}^{1/2} \mathbf{x}_{t+1}$, observe that
    \begin{align*}
        \mathbf{P}^{1/2} \mathbf{x}_{t+1} &= \mathbf{P}^{1/2} \mathbf{x}_t - \eta \mathbf{P}^{-1/2} \mathbf{m}_{t+1} \\
        &= \tilde{\mathbf{x}}_t - \eta \tilde{\mathbf{m}_t} \\
        &= \tilde{\mathbf{x}}_{t+1}.
    \end{align*}
    where the two replacements in the middle step are due to the inductive hypothesis holding for step $t$.

\end{proof}

Due to Proposition \ref{prop:equivalence}, we know that $\textsc{PreconEmaHB}(\eta, \beta_1, \mathbf{P})$ on a quadratic function $$f(\mathbf{x}) = \frac{1}{2} \mathbf{x}^T \mathbf{A} \mathbf{x} + \mathbf{b}^T \mathbf{x} + c$$ has the same stability behavior as  $\textsc{EmaHB}(\eta, \beta_1)$ on the reparameterized quadratic $$\tilde{f}(\mathbf{x}) = \frac{1}{2} \mathbf{x}^T ( \mathbf{P}^{-1/2} \mathbf{A} \mathbf{P}^{-1/2}) \mathbf{x} + \mathbf{b}^T \mathbf{P}^{-1/2} \mathbf{x} + c. $$
Thus, $\textsc{PreconEmaHB}(\eta, \beta_1, \mathbf{P})$ diverges on $f$ whenever the maximum eigenvalue of $( \mathbf{P}^{-1/2} \mathbf{A} \mathbf{P}^{-1/2})$ exceeds the stability threshold of $\textsc{EmaHB}(\eta, \beta_1)$.
When this occurs, the component of $\mathbf{P}^{1/2} \mathbf{x}_t$ that is aligned with $\mathbf{v}_1$, the top eigenvector of the preconditioned Hessian, will oscillate with exponentially increasing magnitude.
Equivalently, we can say that the component of $\mathbf{x}_t$ that is aligned with ${(\mathbf{P}^{1/2})}^T \mathbf{v}_1$ will oscillate with exponentially increasing magnitude.

\clearpage

\section{Adaptive preconditioning}

In Figure \ref{fig:adaptive-preconditioning}, we optimize the objective function $f(x) = \frac{1}{2} x^2$ using rmsprop with learning rates that would be divergent if the preconditioner $\nu_t$ were held fixed at its initial value of $\nu_0 = 1.0$. Yet since the preconditioner is adaptive rather than fixed, rmsprop only oscillates explosively for a finite number of steps before the large gradients associated with this explosive movement cause $\nu_t$ to increase, which decreases rmsprop's effective learning rate and halts divergence.

\begin{figure}[h!]
    \begin{center}
    \includegraphics[width=8cm]{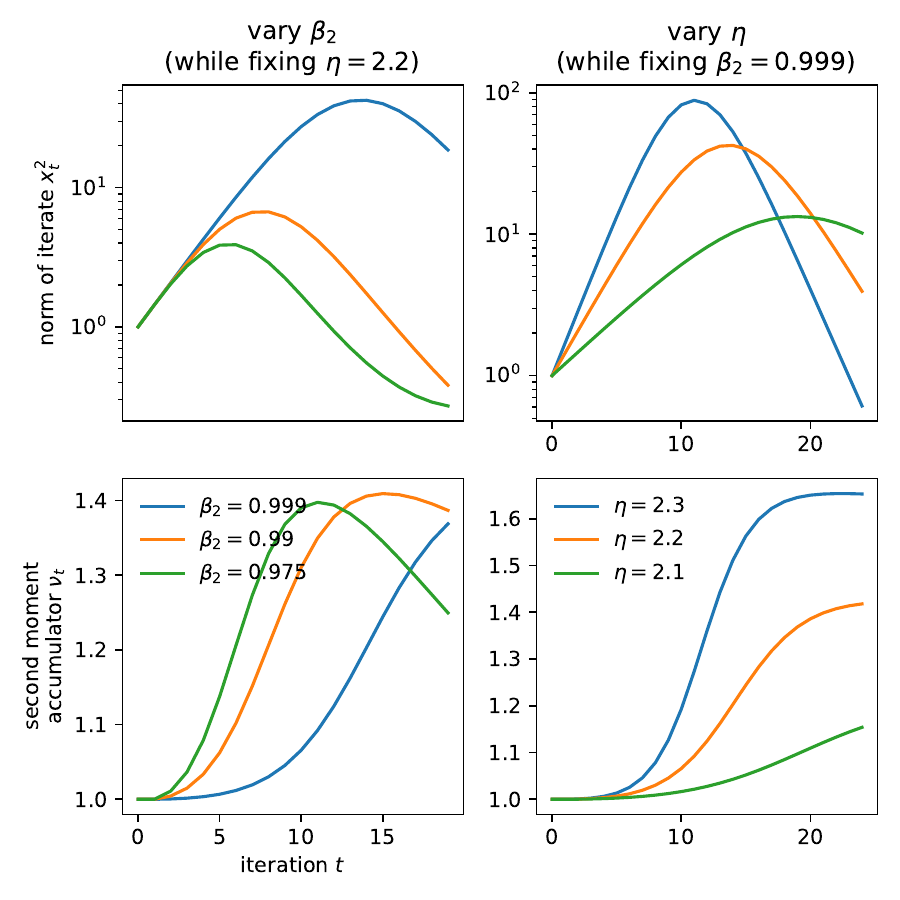}
    \caption{\textbf{Adaptive preconditioning often prevents rmsprop from diverging.}  We optimize the objective function $f(x) = \frac{1}{2} x^2$ using rmsprop.  In the left column, we hold $\eta$ fixed at $\eta=2.2$ and vary $\beta_2$; in the right column, we hold $\beta_2$ fixed at $\beta_2 = 0.999$ and vary $\eta$. Observe that rmsprop only oscillates explosively for a finite number of steps before the large gradients associated with this explosive movement cause $\nu_t$ to increase, which decreases rmsprop's effective learning rate and halts divergence. }
    \end{center}
    \label{fig:adaptive-preconditioning}
\end{figure}

\newpage

\section{Qualitative aspects of the AEoS}

Looking closely at Figure \ref{fig:algorithms}, one can discern that the preconditioned sharpness behaves in qualitatively different ways depending on the type of momentum.
For rmsprop and Adagrad, which do not employ any momentum, the preconditioned sharpness hovers above the stability threshold.
For Nadam, too, which employs Nesterov-style momentum, the preconditioned sharpness also hovers above the stability threshold.
In contrast, for the other algorithms, which employ heavy-ball style momentum, the preconditioned sharpness does not hover above the stability threshold --- instead, the optimizer is rapidly thrown into flatter regions whenever the preconditioned sharpness crosses the threshold; as a result, the preconditioned sharpness oscillates around the stability threshold instead of hovering above it.

We now (1) confirm that these behaviors also occur in the simpler setting of non-adaptive gradient descent, and (2) suggest a speculative explanation.

\paragraph{Non-adaptive gradient descent}
The pattern described above also holds for non-adaptive gradient descent.
For example, Figures 29, 95, and 96 in \citep{cohen2021gradient} depict the evolution of the sharpness as a network is trained using vanilla gradient descent, heavy ball momentum, and Nesterov momentum, respectively.
In both Figure 29 (vanilla) and Figure 96 (Nesterov), the sharpness hovers above the stability threshold, whereas in Figure 95 (heavy ball), the sharpness oscillates around the threshold.

We now reproduce this observation using the same network from Figure \ref{fig:algorithms}.
In Figure \ref{fig:sharpness-momentum}, we train the fully-connected network on CIFAR-10 using vanilla gradient descent $\textsc{GD}(\eta)$, heavy ball gradient descent $\textsc{StandardHB}(\eta, 0.9)$, and Nesterov gradient descent $\textsc{StandardNesterov}(\eta, 0.9)$, each at a range of learning rates.
Observe that for vanilla and Nesterov gradient descent, the sharpness hovers above the stability threshold (except at the smallest learning rate considered), whereas for heavy ball gradient descent, the sharpness oscillates around the stability threshold.

\begin{figure}[h!]
    \includegraphics[width=13cm]{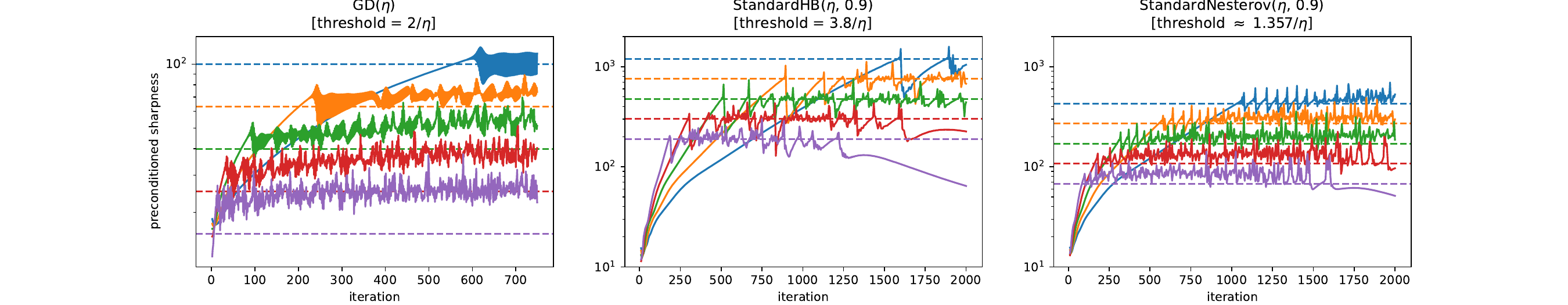}
    \caption{\textbf{Similar to Figure \ref{fig:algorithms}, the behavior of the sharpness during (non-adaptive) gradient descent depends on the type of momentum.}  We train a fully-connected network on CIFAR-10 using three non-adaptive algorithms: vanilla gradient descent ($\eta$ in logspace between 0.02 and 0.2), heavy ball momentum ($\beta_1 = 0.9$, and $\eta$ in logspace between 0.002 and 0.02), and Nesterov momentum ($\beta_1 = 0.9$, and $\eta$ in logspace between 0.002 and 0.02).  Observe that for vanilla GD and Nesterov, the sharpness hovers above the stability threshold, whereas for heavy ball, it oscillates around the stability threshold.  This pattern parallels that in Figure \ref{fig:algorithms} for Adam. }
    \label{fig:sharpness-momentum}
\end{figure}

\paragraph{A speculative explanation}
We now give a speculative explanation for the above observations.
One difference between heavy ball gradient descent, on the one hand, and Nesterov and vanilla gradient descent, on the other, is that heavy ball momentum can be much less robust to increases in the sharpness beyond the optimizer's stability threshold.
By way of background, recall from \citep{cohen2021gradient} that on quadratic functions, all three of these algorithms act independently along each Hessian eigenvector; in particular,
the component of the iterate that is aligned with each Hessian eigenvector evolves according to a linear recurrence relation whose coefficients are determined by the corresponding Hessian eigenvalue $\lambda$ \citep{cohen2021gradient}.
For example, the recurrence relation for $\textsc{GD}(0.2)$ is $x_{t+1} = (1 - 0.2 \, \lambda) x_t$, the recurrence relation for $\textsc{StandardHB}(0.02, 0.9)$ is $x_{t+1} = (1.9 - 0.02 \lambda) x_t - 0.9 x_{t-1}$, and the recurrence relation for $\textsc{StandardNesterov}(0.02, 0.9)$ is $x_{t+1} = 1.9(1 - 0.02 \lambda) x_t - 0.9 (1 - 0.02 \lambda) x_{t-1}$.
Asymptotically (i.e. over many optimizer steps), the growth rate of a linear recurrence relation is determined by the magnitude of the largest root of the characteristic polynomial.
(Note that for a recurrence of the form $x_{t+1} = a x_t + b x_{t-1}$, the characteristic polynomial is $x^2 - ax - b$.)
The stability threshold $\lambda^*$ of an optimizer is the value of $\lambda$ for which the magnitude of the largest-magnitude root is exactly 1.
The optimizer's robustness to barely-unstable eigenvalues ($\lambda$ just above $\lambda^*$) can be quantified by the asymptotic growth rate at at those eigenvalues.

For example, consider the three optimization algorithms $\textsc{GD}(0.2)$, $\textsc{StandardHB}(0.02, 0.9)$, and $\textsc{StandardNesterov}(0.02, 0.9)$.
These algorithms all share the same effective learning rate of 0.2, and they correspond to the purple lines in Figure \ref{fig:sharpness-momentum}.
In Figure \ref{fig:agf}, we visualize the asymptotic growth factor of these algorithms as a function of the input eigenvalue $\lambda$.
The stability threshold $\lambda^*$ is marked with a vertical dotted black line (of course, the asymptotic growth factor at $\lambda^*$ is always 1), and we plot the asymptotic growth factor for eigenvalues up to 1.1 times the stability threshold.

Crucially, observe that the asymptotic growth rate for $\textsc{StandardHB}(0.02, 0.9)$ at an eigenvalue 1.1 times its stability threshold is much higher than for $\textsc{GD}(0.02)$ and $\textsc{StandardNesterov}(0.02, 0.9)$.
This means that whenever $\textsc{StandardHB}$ finds itself in a region of parameter space where the sharpness is 1.1 times its stability threshold, it diverges much faster than $\textsc{GD}(0.02)$ or $\textsc{StandardNesterov}(0.02, 0.9)$ do in an analogous situation.
One can easily imagine that this may be related to our observation that the latter two algorithms can train at the AEoS with the sharpness hovering a bit above the stability threshold, whereas heavy ball momentum is flung into a flatter region whenever the sharpness crosses the stability threshold even by a small amount.

\begin{figure}[h!]
    \includegraphics[width=13cm]{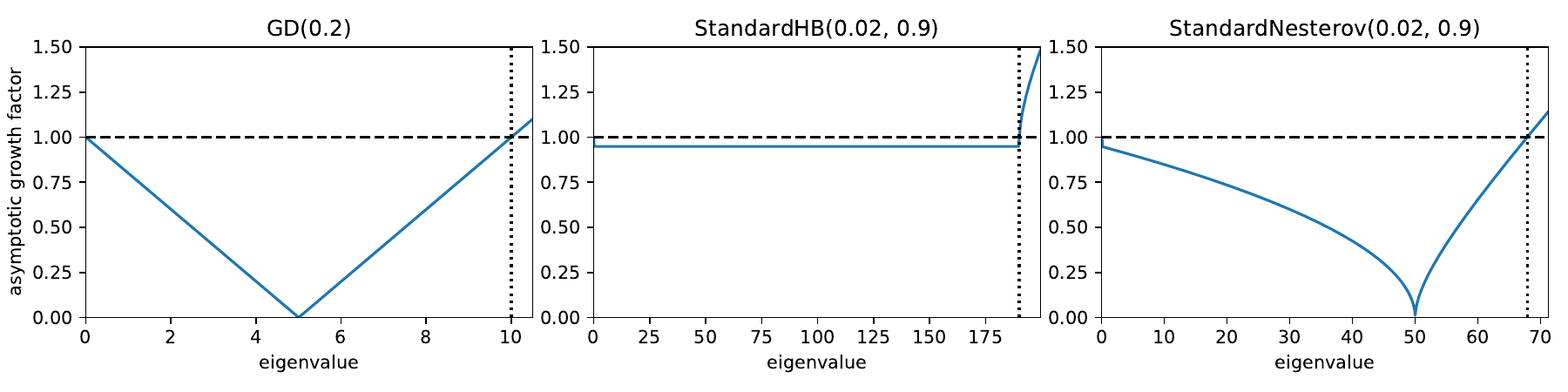}
    \caption{\textbf{Vanilla gradient descent and Nesterov momentum are both more robust than heavy ball momentum to eigenvalues shortly above the stability threshold.}  We plot the asymptotic growth factors of three algorithms: vanilla GD, heavy ball GD, and Nesterov GD.  That is, for each possible Hessian eigenvalue, we plot the multiplicative constant which the corresponding component of the iterate will be asymptotically multiplied by, over many steps of gradient descent.   The vertical dotted line marks each algorithm's stability threshold, and we plot the asymptotic growth factor (AGF) for eigenvalues up to 1.1 times that stability threshold.  Observe that for heavy ball momentum, the AGF at 1.1 times the stability threshold is much higher than for Nesterov momentum or vanilla gradient descent.}
    \label{fig:agf}
\end{figure}

\clearpage

\section{Corner case}

While we generally observe that during full-batch Adam, the preconditioned sharpness equilibrates at the stability threshold of frozen Adam, we have also consistently observed one notable corner case where this does not occur.
Namely, when running Adam at \emph{extremely small} learning rates, we often find that the preconditioned sharpness either flatlines short of the stability threshold, or increases at a very slow rate while remaining below the threshold.
These behaviors are evident in Figure \ref{fig:cornercase}, in which we train a CNN on CIFAR-10 using full-batch Adam at a range of very small learning rates.
(This task is the same one as in Figure \ref{fig:architectures}, left panel, just at smaller learning rates.)
Observe that at these small learning rates, the preconditioned sharpness fails to rise to the stability threshold.

\begin{figure}[h!]
    \includegraphics[width=12cm]{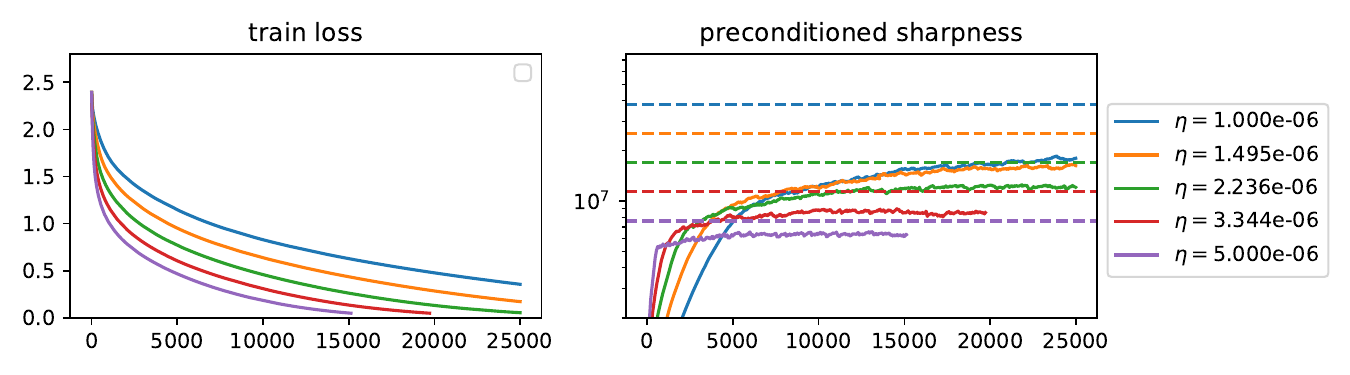}
    \caption{\textbf{When running Adam at \emph{extremely small} learning rates, the sharpness can fail to rise all the way to the stability threshold.} We train a CNN on CIFAR-10 using full-batch Adam at a range of extremely small learning rates.  Observe that the preconditioned sharpness fails to rise all the way to the stability threshold.}
    \label{fig:cornercase}
\end{figure}

Why does the preconditioned sharpness not rise to the stability threshold?
One clue is that this behavior seems to be related to adaptivity.
In Figure \ref{fig:cornercase-freeze}, we train using full-batch Adam at the green learning rate from Figure \ref{fig:cornercase}, and at step 6000 of training, we suddenly freeze the preconditioner (i.e. we stop updating $\boldsymbol{\nu}_t$).  
The train loss and preconditioned sharpness of frozen Adam are plotted in pink.
Observe that upon freezing the preconditioner, the preconditioned sharpness rapidly rises all the way to the stability threshold.
Thus, we can infer that it was due to Adam's adaptivity that the preconditioned sharpness was located below the stability threshold in the first place.

\begin{figure}[h!]
    \includegraphics[width=12cm]{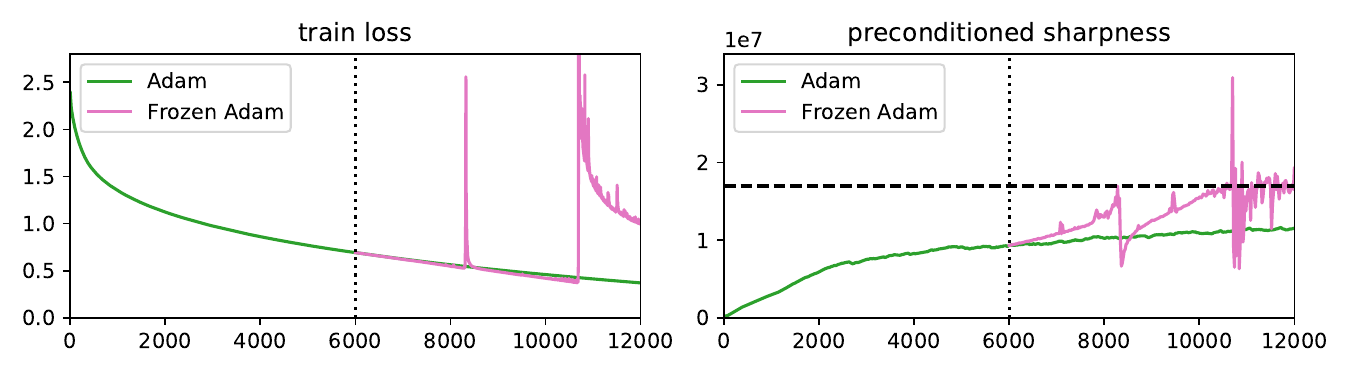}
    \caption{\textbf{If we freeze the preconditioner, the preconditioned sharpness immediately rises to the stability threshold.}  In green, we train a CNN on CIFAR-10 using full-batch Adam at step size $\eta = 2.236$e-6 (the green learning rate from Figure \ref{fig:cornercase}).  In pink, we freeze the preconditioner at step 6,000.  Observe that once the preconditioner is frozen, the preconditioned sharpness rises immediately to the stability threshold, suggesting that adaptive preconditioning was preventing it from rising to the stability threshold in the first place.}
    \label{fig:cornercase-freeze}
\end{figure}

We hypothesize that this behavior is related to the following property of momentum algorithms: momentum algorithms can generate large gradients even when they are stable.
For example, in Figure \ref{fig:momentum-large-gradients}, we optimize the quadratic objective $f(x) = \frac{1}{2} x^2$ using EMA heavy ball momentum at $\beta_1 = 0.9$ and three different stable learning rates (columns).
We set the initial iterate $x_0$ to 1, and the initial momentum vector $m_0$ to 0.2.
We run $\textsc{EmaHB}(\eta, 0.9)$ for thirty steps, plotting the iterate (top row), the momentum vector (middle row), and the gradient squared norm  (bottom row).
Observe that the gradient squared norm increases dramatically from its initial value.

We therefore hypothesize that the preconditioned sharpness $\lambda_1(\mathbf{P}_t^{-1} \mathbf{H}_t)$ fails to rise all the way to the stability threshold because the large gradients generated by the momentum algorithm (even while it is stable) are causing the preconditioner $\mathbf{P}_t$ to constantly increase.  This, in turn ``cancels out'' the effect of progressive sharpening, and causes the preconditioned sharpness $\lambda_1(\mathbf{P}_t^{-1} \mathbf{H}_t)$ to remain short of the stability threshold.

\begin{figure}[h!]
    \includegraphics[width=14cm]{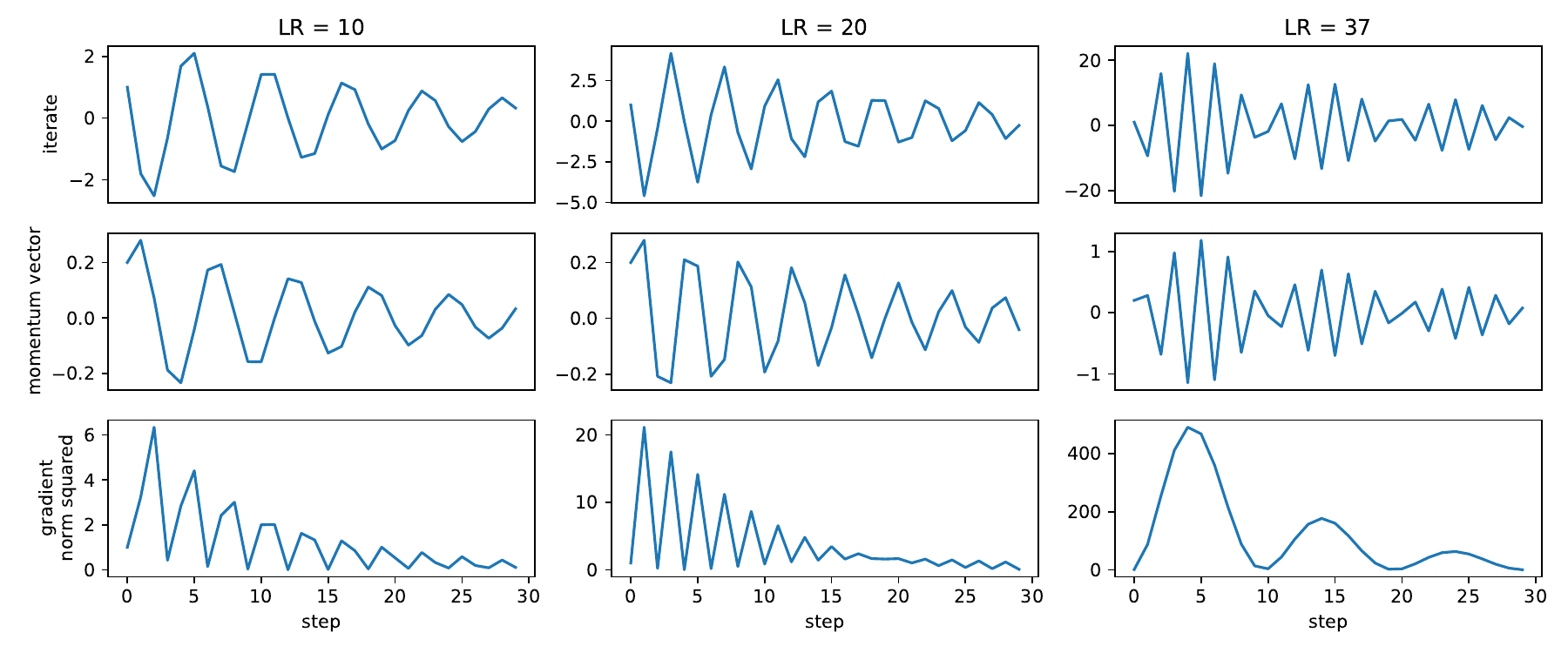}
    \caption{\textbf{Momentum can generate large gradients even when stable.}  We optimize the function $f(x) = \frac{1}{2} x^2$ using 30 steps of $\textsc{EmaHB}(\eta, 0.9)$ at various stable learning rates $\eta$. The top row plots the iterate, the second row the momentum vector (really a scalar), and the third row the squared gradient norm.  Observe that the gradient norm increases from its initial value, sometimes dramatically.  This does not happen with vanilla gradient descent.}
    \label{fig:momentum-large-gradients}
\end{figure}

\clearpage

\section{Further plots of zoom in}

\begin{figure}[h!]
    \includegraphics[width=14cm]{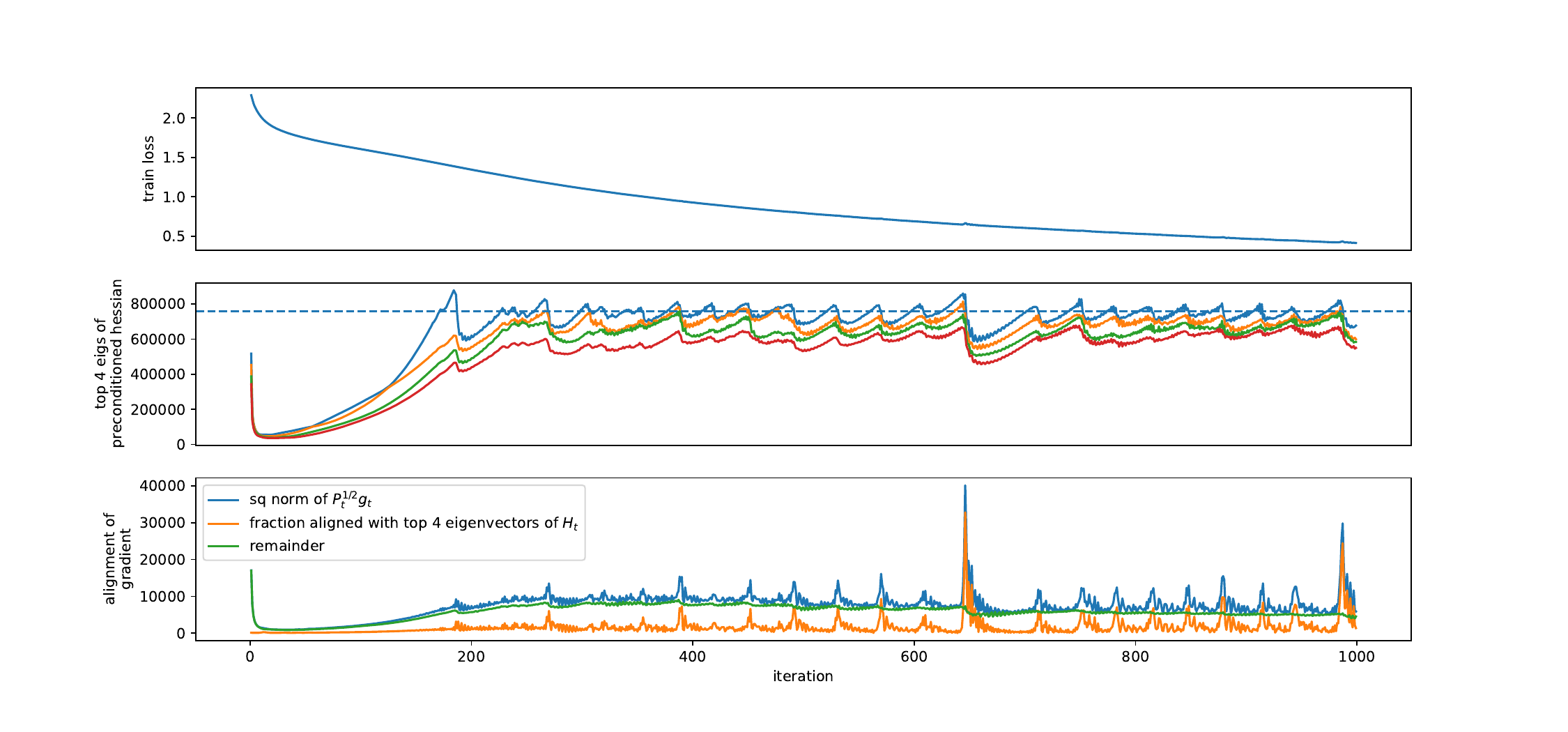}
    \caption{\textbf{Explosive growth in the (preconditioned) gradient is limited to the top eigenvectors of the preconditioned Hessian.}  We train a fully-connected network on CIFAR-10 using full-batch Adam.  In the top row, we plot the train loss.  In the middle row, we plot the evolution of the top four eigenvalues of the preconditioned Hessian $\mathbf{P}_t^{-1} \mathbf{H}_t$.  In the bottom row, we plot (1) in blue, the squared norm of the ``semi-preconditioned gradient'' $\mathbf{P}^{1/2} \mathbf{g}_t$, which is the quantity that is expected to grow explosively when the algorithm is unstable; in orange, the squared norm of the portion of  $\mathbf{P}^{1/2} \mathbf{g}_t$ that is aligned with the top four eigenvectors of $\mathbf{P}_t^{-1} \mathbf{H}_t$; and in green, the squared norm of the remainder of $\mathbf{P}^{1/2} \mathbf{g}_t$.  Observe that almost all of the spikes in the gradient norm can be attributed to the top eigenvectors of the preconditioned Hessian.}
    \label{fig:alignment}
\end{figure}

\newpage

\section{Test accuracy of Wide ResNet}

In \S \ref{sec:explain}, we saw that Adam tended to find solutions with higher sharpness (maximum Hessian eigenvalue) whenever the learning rate $\eta$ was smaller or the decay factor $\beta_2$ was smaller.
We now demonstrate that the test error tends to be higher in these same situations.
In Figure \ref{fig:wrn-test-error}, we plot the test error of the Wide ResNet from Figure  \ref{fig:hyperparameters} as a function of $\eta$ and $\beta_2$.
(For each training run, we select the test error at the moment the train loss first drops below the threshold of 0.05.)
As in Figure \ref{fig:hyperparameters}, we experiment with both the full-batch and large batch (batch size 4096) settings.
Observe that across both these settings, the test error tends to be higher whenever $\eta$ is small (subplots (a) or (b)) or whenever $\beta_2$ is small (subplots (c) or (d)).
We emphasize that we do not intend to claim that there is necessarily a causal connection between the test error and the sharpness --- we only mean to point out that the Adam's implicit bias towards low-curvature regions when $\eta$ is high or $\beta_2$ is high coincides with superior generalization performance.

\begin{figure}[h!]
    \includegraphics[width=14cm]{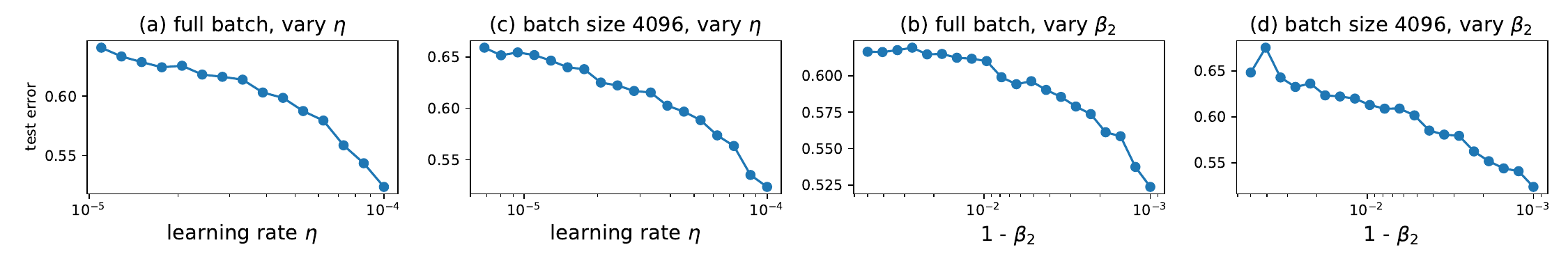}
    \caption{\textbf{Adam tends to generalize better when $\eta$ or $\beta_2$ is small.}  We train a Wide ResNet on CIFAR-100 using full-batch (panels a, c) or large-batch (panels b, d) Adam at various $\eta$ (panels a, b) or $\beta_2$ (panels b, d).  We plot the test error, measured when the training loss first drops below the threshold 0.05. }
    \label{fig:wrn-test-error}
\end{figure}

\newpage

\section{Experimental details}

\subsection{Architectures}

\paragraph{Fully-connected network} This network has 5 hidden layers of width 200 each.  The activation function is tanh, and the weights are initialized with variance 1 / fan\_in.
The fully-connected layers have bias parameters, which are initialized to zero.

\paragraph{CNN}
This network repeats (Conv => ReLU => maxpool) three times, then flattens the activations and does FC => ReLU => FC => ReLU => FC.
The first conv layer has 64 filters, a kernel size of 5, ``valid'' padding, a 3x3 window size, and stride 2.
The second conv layer has 96 filters, a kernel size of 3, ``valid'' padding, a 3x3 window size, and stride 2.
The third conv layer has 128 filters, a kernel size of 3, ``same'' padding, a 3x3 window size, and stride 2.
The first fully-connected layer has width 512, and the second has width 256.
All biases are initialized to zero, and all weights are initialized with variance 1 / fan\_in.

\paragraph{Wide ResNet (normalized)}
This Wide ResNet has four blocks per group, and a channel multiplier of 10.

\paragraph{Wide ResNet (un-normalized)}
When we remove batch normalization, we insert a learnable scalar multiplier (initialized to 1) after the residual branch of each block.

\subsection{Figures}

\paragraph{Figure 2}
We extracted the fixed preconditioner from step 1100 of a network trained using full-batch Adam at learning rate 2e-5.

\paragraph{Figure 4}
In the left pane, we trained for 2000 steps or until reaching a threshold loss value of 0.05.
In the middle pane, we trained for 5000 steps.
In the right pane, we trained until reaching a threshold loss value of 0.05.

\paragraph{Figure 5}
\textbf{Adam with bias correction}: $\beta_1 = 0.9$, $\beta_2 = 0.999$, $\epsilon=$1e-7, learning rates were:

\texttt{np.logspace(np.log10(3e-5), np.log10(1e-3), 5)}

\textbf{AdamW}: $\beta_1 = 0.9$, $\beta_2 = 0.999$, $\epsilon=$1e-7, learning rates were:

\texttt{np.logspace(np.log10(3e-5), np.log10(1e-3), 5)}

\textbf{Adafactor}: The decay rate was 0.8, the momentum was 0.9, epsilon was 1e-7, and the learning rates were

\texttt{np.logspace(-5, -3, 5)}

\textbf{Amsgrad}: $\beta_1 = 0.9$, $\beta_2 = 0.999$, $\epsilon=$1e-7, bias correction was off, and learning rates were:

\texttt{np.logspace(-5, -3, 5)}

\textbf{Padam}: $\beta_1 = 0.9$, $\beta_2 = 0.999$, $\epsilon=$1e-7, bias correction was off, the Padam exponent was 0.25, and learning rates were:

\texttt{np.logspace(-3, -1, 5)}

\textbf{Nadam}: $\beta_1 = 0.9$, $\beta_2 = 0.999$, $\epsilon=$1e-7, bias correction was off, and learning rates were:

\texttt{np.logspace(-5, -3, 5)}

\textbf{rmsprop}: decay was 0.995, epsilon was 1e-7, and learning rates were:

\texttt{np.logspace(-4, -3, 5)}

\textbf{Adagrad}: Learning rates were \texttt{np.logspace(-2, -1, 5)}.

\paragraph{Figure 10}
For the learning rate sweeps we fixed $\beta_2 = 0.999$, and for the $\beta_2$ sweeps we fixed $\eta =$ 1e-4.

\paragraph{Figure 11}
For Adam we used $\beta_1 = 0.9$ and $\beta_2 = 0.999$; for momentum we used $\beta_1 = 0.9$.

\end{document}